%% file: main.tex
\documentclass[sigconf]{aamas}  
\usepackage{balance}  

\usepackage{booktabs}       
\usepackage{amsfonts}       
\usepackage{nicefrac}       
\usepackage{microtype}      
\usepackage{amsmath}
\usepackage{amssymb}
\usepackage{amsthm}
\usepackage{bm}
\usepackage{bbm}
\usepackage{algorithm}
\usepackage[noend]{algpseudocode}
\usepackage{subcaption}
\usepackage{graphicx}
\usepackage{float}
\usepackage{wrapfig}

\setcopyright{ifaamas}  
\acmDOI{}  
\acmISBN{}  
\acmConference[AAMAS'19]{Proc.\@ of the 18th International Conference on Autonomous Agents and Multiagent Systems (AAMAS 2019)}{May 13--17, 2019}{Montreal, Canada}{N.~Agmon, M.~E.~Taylor, E.~Elkind, M.~Veloso (eds.)}  
\acmYear{2019}  
\copyrightyear{2019}  
\acmPrice{}  

\allowdisplaybreaks

\algnewcommand{\IfThenElse}[3]{
  \State \algorithmicif\ #1\ \algorithmicthen\ #2\ \algorithmicelse\ #3}
\DeclareMathOperator{\rad}{\mathit{rad}}

\DeclareMathOperator{\doracle}{\mathit{Oracle}}
\DeclareMathOperator{\Htilde}{\tilde{\mathbf{H}}}
\DeclareMathOperator{\spp}{\text{spp}}

\newcommand{\distance}[2]{\mathit{dist}({#1},{#2})}
\newcommand{\ceil}[1]{\mathit{ceil}({#1})}
\newcommand{\floor}[1]{\mathit{floor}({#1})}
\newcommand{\width}{\mathit{width}}
\newcommand{\wdiverse}{w_{\textsc{div}}}
\newcommand{\wtop}{w_{\textsc{top}}}
\newcommand{\Mdiverse}{M^*_{\textsc{div}}}
\newcommand{\Mtop}{M^*_{\textsc{top}}}
\newcommand{\Cost}{\mathit{Cost}}
\newcommand{\Gain}{\mathit{Gain}}

\begin{document}

\title{The Diverse Cohort Selection Problem} 

\author{Candice Schumann}
\affiliation{%
  \institution{University of Maryland}
}
\email{schumann@cs.umd.edu}

\author{Samsara N. Counts}
\affiliation{%
  \institution{George Washington University}
}
\email{countss@gwmail.gwu.edu}

\author{Jeffrey S. Foster}
\affiliation{%
  \institution{Tufts University}
}
\email{jfoster@cs.tufts.edu}

\author{John P. Dickerson}
\affiliation{%
  \institution{University of Maryland}
}
\email{john@cs.umd.edu}

 
\begin{abstract}
\input{abstract}
\end{abstract}

\keywords{AAMAS; ACM proceedings; organisations and institutions; social choice theory}
 
\maketitle

\begin{quote}{\small\emph{``It should come as no surprise that more diverse companies and institutions are achieving better performance.'' -- McKinsey \& Company, \emph{Diversity Matters} (2015)}}
\end{quote}

\input{introduction}

\input{related_works}

\input{problem_formulation}

\input{algorithms}

\input{experiments}

\input{discussion}

\input{conclusion}

\section{Acknowledgements}
 Schumann and Dickerson were supported by NSF IIS RI CAREER Award \#1846237; Counts was supported by NSF REU-CAAR (Combinatorics and Algorithms for Real Problems) CNS \#1560193 hosted at the University of Maryland.  We thank Google for gift support, and the anonymous reviewers for helpful comments.

\bibliographystyle{ACM-Reference-Format}  
\balance  
\bibliography{refs,cohort}

\clearpage
\input{appendix}

\end{document}

%% file: abstract.tex
How should a firm allocate its limited interviewing resources to select the optimal cohort of new employees from a large set of job applicants?  How should that firm allocate cheap but noisy resume screenings and expensive but in-depth in-person interviews?  We view this problem through the lens of combinatorial pure exploration (CPE) in the multi-armed bandit setting, where a central learning agent performs costly exploration of a set of arms before selecting a final subset with some combinatorial structure.  We generalize a recent CPE algorithm to the setting where arm pulls can have different costs and return different levels of information. We then prove theoretical upper bounds for a general class of arm-pulling strategies in this new setting.  We apply our general algorithm to a real-world problem with combinatorial structure: incorporating diversity into university admissions.  We take real data from admissions at one of the largest US-based computer science graduate programs and show that a simulation of our algorithm produces a cohort with hiring overall utility while spending comparable budget to the current admissions process at that university.

%% file: introduction.tex
\section{Introduction}\label{sec:introduction}

How should a firm, school, or fellowship committee allocate its limited interviewing resources to select the optimal cohort of new employees, students, or awardees from a large set of applicants?  Here, the central decision maker must first form a belief about the true quality of an applicant via costly information gathering, and then select a subset of applicants that maximizes some objective function.  Furthermore, various types of information gathering can be performed---reviewing a r{\'e}sum{\'e}, scheduling a Skype interview, flying a candidate out for an all-day interview, and so on---to gather greater amounts of information, but also at greater cost.  



In this paper, we model the allocation of structured interviewing resources and subsequent selection of a cohort as a combinatorial pure exploration problem in the multi-armed bandit (MAB) setting.  Here, each applicant is an arm, and a decision maker can \emph{pull} the arm, at some cost, to receive a noisy signal about the underlying quality of that applicant.  We further model two different levels of interviews as \emph{strong} and \emph{weak} pulls---the former costing more to perform than the latter, but also resulting in a less noisy signal.  We introduce the strong-weak arm-pulls (SWAP) algorithm, generalizing an algorithm by~\citet{chen2014combinatorial}, and provide theoretical upper bounds for a general class of our various arm-pull strategies. To complement these bounds, we provide simulation results comparing pulling strategies on a toy problem that mimics our theoretical assumptions.

We then validate our proposed method on a real-world scenario: admitting an optimal cohort of graduate students.  We take recent data from one of the largest US-based Computer Science graduate programs---applications including recommendation letters, statements of purpose, transcripts, as well as the department's reviews of applications and final admissions decisions---and run experiments comparing our algorithm's performance under a variety of assumptions to reviews and decisions made in reality. 
We find that our simulation of SWAP produced a cohort with higher top-K utility using equivalent resources as in practice.

We also explore the empirical performance of our algorithm optimizing a nonlinear objective function, motivated by the real-world scenario of admitting a diverse cohort of graduate students. In experiments,  our simulations of SWAP increased a diversity score (over gender and region of origin) with little loss in fit using roughly the same amount of resources as in practice. This gain suggests that SWAP can serve as a useful decision support tool to promote diversity in practice.


%% file: related_works.tex
\section{Related Work}
\label{relatedwork}

The multi-armed bandit (MAB) problem is a classic setting for modeling sequential decision making;~\citet{Bubeck12:Regret} provide an in-depth overview. 
Previous work in the MAB setting has looked at selecting a subset of arms to maximize some objective.
Other work focuses on varied rewards from and costs of pulling arms. To the best of our knowledge, no work operates at the intersection of these two spaces. 
\citet{chen2014combinatorial} provide a general formulation of top-K multi-armed bandits in the combinatorial setting. They provide both a fixed confidence and a fixed budget algorithm.
Our work builds on these contributions by adding varied---in terms of cost and reward---arm pulls.

Several MAB formulations select an optimal subset using a \emph{single} type of arm pull, modeling decisions with focuses on different problem features.
\citet{cao2015top} solve the top-K problem with MABs for linear objectives.
\citet{locatelli2016optimal} address the thresholding bandit problem, finding the arms above and below threshold $\tau$ with precision~$\epsilon$. 
\citet{jun2016top} identify the top-K set while pulling arms in batches.
\citet{singla2015learning} propose an algorithm for crowdsourcing that hires a team for specific tasks, treating types of workers as separate problems and an arm pull as a worker performing an action with uniform cost.

To select the best subset while satisfying a submodular function, \citet{singla2015noisy} propose an algorithm maximizing an unknown function accessed through noisy evaluations. \citet{radlinski2008learning} learn a diverse ranking from the behavior patterns of different users and then greedily select the next document to rank. They treat each rank as a separate MAB instance, rather than our approach using a single MAB to model the whole system. 
\citet{yue2011linear} introduce the \emph{linear submodular bandits problem} to select diverse sets of content in an online learning setting for optimizing a class of feature-rich submodular utility models. 

We are motivated by the observation that, in many real-world settings, different levels of information gathering can be performed at varying costs. Previous work uses stochastic costs in the MAB setting. However, our costs are fixed for specific types of arm pulls. 
\citet{ding2013multi} look at a MAB problem with variable rewards and cost with budget constraints. When an arm is pulled, a random reward is received, and a random cost is taken from the budget.
Similarly, \citet{xia2016budgeted} propose a batch-arm-pull MAB solution to a problem with variable, random rewards and costs.
\citet{jain2014incentive} use MABs with variable rewards and costs to select individual workers in a crowdsourcing setting. They select workers to do binary tasks with an assured accuracy for each, where workers' costs are unknown. 

\citet{LuxUgradSupervisedAdmit} and \citet{WatersGRADE2013} 
use supervised learning to model admissions decisions. They develop accurate classifiers; none decide how to allocate interviewing resources or maximize a certain objective, unlike our aim to select a more diverse cohort via a principled semi-automated system.

The behavioral science literature shows that scoring candidates via the same rubric, asking the same questions, and spending the same amount of time are interviewing best practices
~\cite{Schmitt76:Social,Arvey82:Employment,Harris89:Reconsidering, Williamson97:Employment}. 
Such \emph{structured interviews} reduce bias and provide better job success predictors
~\cite{Posthuma02:Beyond,Levashina14:Structured}.
We incorporate these results into our model through our assumption that we can spend the same budget and get the same information gain across different arms.

%% file: problem_formulation.tex
\section{Problem Formulation}\label{sec:model}

We now formally describe the stochastic multi-armed bandit setting in which we operate.  For exposition's sake, we do so in the context of a decision-maker reviewing a set of job applicants.
However, the formulation itself is fully general.
We represent a set of $n$ applications $A$ as arms $a_i \in A$ for $i \in [n]$. Each arm has a true utility, $u(a_i)\in [0,1]$, which is unknown; an empirical estimate $\hat{u}(a_i)\in [0,1]$ of that underlying true utility; and an uncertainty bound $\rad(a_i)$. Once arm $a_i$ is pulled (e.g., application reviewed or applicant interviewed), $\hat{u}(a_i)$ and $\rad(a_i)$ are updated.

The set of potential \emph{cohorts}, or subsets of arms, is defined by a decision class $\mathcal{M}\subseteq 2^{[n]}$. Note that $\mathcal{M}$ need not be the power set of arms, but can include cardinality and other constraints. The total utility for a cohort is given by some linear function $w:\mathbb{R}^n\times\mathcal{M}\rightarrow\mathbb{R}$ that takes as input the (unknown) true utilities $u(\cdot)$ of the arms and the selected cohort. Throughout the paper, we assume a maximization oracle, defined as $\doracle(\mathbf{v})={\arg\max}_{M\in \mathcal{M}}w(M)$,
where $\mathbf{v}\in\mathbb{R}^n$ is a vector of weights---in this case, estimated or true utilities for each arm. Our overall goal is to accurately estimate the true utilities of arms and then select the optimal subset of arms using the maximization oracle.

\paragraph{Problem hardness.}
Following the notation of~\citet{chen2014combinatorial}, we define a \emph{gap} score for each arm. For each arm $a$ that is in the optimal cohort $M^*$, the gap is the difference in optimality between $M^*$ and the best set without $a$. 
For each arm $a$ that is not in the optimal set $M^*$, the gap is the sub-optimality of the best set that includes  $a$. Formally, the gap is defined as
\begin{equation}\label{eq:gap}
\Delta_a=
    \begin{cases}
    	w(M^*)-\max_{M\in\mathcal{M}:a\in M} w(M) ,& \text{if } a\notin M^*\\
        w(M^*)-\max_{M\in\mathcal{M}:a\notin M} w(M) ,& \text{if } a\in M^*.\\
    \end{cases}
\end{equation}

This gap score serves as a useful signal for problem hardness, which we use in our theoretical analysis.  Formally, the hardness of the problem can be defined as the sum of inverse squared gaps
\begin{equation} \label{eq:hardness}
\mathbf{H}=\sum_{a\in A} \Delta_a^{-2}.
\end{equation}

Chen et al. defined the concept of $\width(\mathcal{M})$. When comparing all combinations of two sets $A,A'\in \mathcal{M}$, where $A \neq A'$, define  $\distance{A}{A'}=|A-A'|+|A'-A|$. Therefore, define  $\width(\mathcal{M}) = \min_{\{A,A' | A,A' \in \mathcal{M} \land A \neq A'\}} \distance{A}{A'}$. In other words, the width is the smallest distance between any two sets in $\mathcal{M}$. See Chen et al. for an in-depth explanation of $\width(\mathcal{M})$.

\paragraph{Strong and weak pulls.}
In reality, there is more than one way to gather information or receive rewards. Therefore, we introduce two kinds of arm pulls which vary  in cost $j$ and information gain $s$. Information gain $s$ is defined as how sure one is the reward is close to the true utility.  We model the information gain as $s$ parallel arm pulls with the resulting rewards being averaged together. A \emph{weak arm pull} has cost $j = 1$ but results in a small amount of information $s =1$. In our domain of graduate admissions, weak arm pulls are standard application reviews, which involve reading submitted materials and then making a recommendation. A \emph{strong arm pull}, in contrast, has cost $j > 1$, but results in $s > 1$ times the information as a weak arm pull. In our domain, strong arm pulls extend reading submitted materials with a structured Skype interview, followed by note-taking and a recommendation.

In our experience, the latter can reduce uncertainty considerably, which we quantify and discuss in Section~\ref{sec:experiments}. However, due to their high cost, such interviews are allocated relatively sparingly.  We formally explore this problem in Section~\ref{sec:alg} and provide an algorithm for selecting which arms to pull, along with nonasymptotic upper bounds on total cost.

%% file: algorithms.tex
\section{SWAP: An Algorithm for Allocating Interview Resources}\label{sec:alg}
In this section, we propose a new multi-armed bandit algorithm, strong-weak arm-pulls (SWAP), that is parameterized by $s$ and $j$. SWAP uses a combination of strong and weak arm pulls to gain information about the true utility of arms and then selects the optimal cohort. Our setting and the algorithm we present generalize the CLUCB algorithm proposed by~\citet{chen2014combinatorial}, which can be viewed as a special case with $s=j=1$.

\newlength\myindent 
\setlength\myindent{6em} 
\newcommand\bindent{%
  \begingroup 
  \setlength{\itemindent}{\myindent} 
  \addtolength{\algorithmicindent}{\myindent} 
}
\newcommand\eindent{\endgroup} 

\begin{algorithm}[ht]
\caption{Strong Weak Arm Pulls (SWAP)}\label{alg:swap}
\begin{algorithmic}[1]
\Require Confidence $\delta\ \in (0,1)$; Maximization oracle: $\doracle(\cdot):\mathbb{R}^n \rightarrow \mathcal{M}$
\State Weak pull each arm $a \in [n]$ once to initialize empirical means $\hat{\mathbf{u}}_n$
\State $\forall i \in [n]$ set $T_n(a_i)\gets 1$,
\State $\Cost_n \gets n$, total resources spent
\For{$t=n,n+1,\ldots$}
\State $M_t \gets \doracle(\hat{\mathbf{u}}_t)$
\For{$a_i = 1,\ldots,n$}
\State $\rad_t(a_i)=\sigma\sqrt{2\log\left(\frac{4n\Cost^3_t}{\delta}/T_t(a_i)\right)}$
\If {$a_i \in M_t$} 
\State $\tilde{u_t}(a_i) \gets \hat{u}_t(a_i) - \rad_t(a_i)$
\Else {} 
\State $\tilde{u_t}(a_i) \gets \hat{u}_t(a_i) + \rad_t(a_i)$
\EndIf
\EndFor
\State $\tilde{M}_t \gets \doracle(\tilde{\mathbf{u}}_t)$
\If {$w(\tilde{M}_t) = w(M_t)$}\label{line:equality}
\State $\texttt{Out} \gets M_t$
\State \Return $\texttt{Out}$
\EndIf
\State $p_t \gets \arg\max_{a \in (\tilde{M}_t \setminus M_t) \cup (M_t \setminus \tilde{M}_t)}\rad_t(a)$
\State $\alpha \gets spp(s,j)$
\State \textbf{with} probability $\alpha$ \textbf{do}
\State \hspace{\algorithmicindent} Strong pull $p_t$
\State \hspace{\algorithmicindent} $T_{t+1}(p_t) \gets T_t(p_t)+s$
\State \hspace{\algorithmicindent}$\Cost_{t+1} \gets \Cost_t+j$
\State \textbf{else}
\State \hspace{\algorithmicindent} Weak pull $p_t$
\State \hspace{\algorithmicindent} $T_{t+1}(p_t) \gets T_t(p_t)+1$
\State \hspace{\algorithmicindent} $\Cost_{t+1} \gets \Cost_t+1$
\State Update empirical mean $\hat{\mathbf{u}}_{t+1}$ using observed reward
\State $T_{t+1}(a) \gets T_t(a)\ \forall a \neq p_t$

\EndFor
\end{algorithmic}
\end{algorithm}

Algorithm \ref{alg:swap} gives pseudocode for SWAP.  It starts by weak pulling all arms once to initialize an empirical estimate of the true underlying utility of each arm.  It then iteratively pulls arms, chooses to weak or strong pull based on a general strategy, updates empirical estimates of arms, and terminates with the optimal (i.e., objective-maximizing) subset of arms with probability $1-\delta$, for some user-supplied parameter $\delta$.  

During each iteration $t$, SWAP starts by finding the set of arms $M_t$ that, according to current empirical estimates of their means, maximizes the objective function via an oracle. It then computes a confidence radius, $\rad_t(a)$, for each arm $a$ and estimates the worst-case utility of that arm with the corresponding bound. If an arm $a$ is in the set $M_t$ then the worst case is when the true utility of $a$ is less than our estimate ($a$ might not be in the true optimal set $M^*$). Alternatively, if an arm is not in the set $M_t$ then the worst case is when the true utility of $a$ is greater than our estimate ($a$ might be in the true optimal set $M^*$).
Using the worst-case estimates, SWAP computes an alternate subset of arms  $\tilde{M}_t$. If the utility of the initial set $M_t$ and the worst-case set $\tilde{M}_t$ are equal, then SWAP terminates with output $M_t$, which is correct with probability $1-\delta$ as we show in Theorems \ref{thm:strong_only} and \ref{thm:swap}. If $w(M_t)$ and $w(\tilde{M}_t)$ differ, SWAP looks at a set of candidate arms in the symmetric difference of $M_t$ and $\tilde{M}_t$ and chooses the arm $p_t$ with the largest uncertainty bound $\rad_t(p_t)$.

SWAP then chooses to either strong or weak pull the selected arm $p_t$ using a \emph{strong pull policy}, depending on parameters $s$ and $j$. A strong pull policy is defined as $\mathit{spp} : \mathbb{R} \geq 1 \times (\mathbb{R} \geq 1) \rightarrow [0,1]$. For example, in the experiments in Section~\ref{sec:experiments}, we use the following pull policy:
\begin{equation}\label{eq:spp}
\mathit{spp}(s,j)=\frac{s-j}{s-1}.
\end{equation}

This policy tries to balance information gain and cost. When the strong pull gain is high relative to cost then many more strong pulls will be performed. When the weak pull gain is low relative to cost then fewer strong pulls will be performed, as discussed in Example~\ref{ex:swap}.

Once an arm is pulled, the empirical mean $\hat{u}_{t+1}(p_t)$ and the information gain $T_{t+1}(p_t)$ is updated. A reward from a strong arm is counted $s$ times more than a weak pull.

\begin{example}\label{ex:swap}
Suppose we wish to find a cohort of size $K=2$ from three arms $A = \{a_1,a_2,a_3\}$. Run SWAP for $t$ iterations. Figure~\ref{fig:ex_swap} shows that SWAP maintains empirical utilities $\hat{u}_t(\cdot)$ and uncertainty bounds $\rad_t(\cdot)$.
In this case $M=\{a_1,a_2\}$ and $\tilde{M}=\{a_1,a_3\}$. Arm $a_3$, therefore, is the arm in the symmetric difference $\{a_2,a_3\}$ with the highest uncertainty, which therefore needs to be pulled. Further, assume that $a_3$ needs $x$ information gain for SWAP to end. When $j=1$ and $s=1$, the best pulling strategy would be to weak pull $a_3$ for $x$ times. When $j=1$ and $s=y$ where $y>1$, the best pulling strategy would be to strong pull $a_3$ for $\ceil{\frac{x}{y}}$ times. Finally when $j=z$ and $s=y$ where $y>z>1$, the best pulling strategy would be to strong pull $a_3$ for $\floor{\frac{x}{y}}+\mathbf{1}[z-(x \mod y)]$ times and weak pull $a_3$ for $\mathbf{1}[z-(x \mod y)]*(x \mod y)$ times, where $\mathbf{1}[a]=1$ when $a\geq 0$ and $0$ otherwise. In reality, we do not know how many times an arm needs to be pulled, which is why we introduce a probabilistic strong pull policy, like that in Equation~\ref{eq:spp}.
\end{example}

\begin{figure}
  \centering
  \includegraphics[width=0.43\columnwidth]{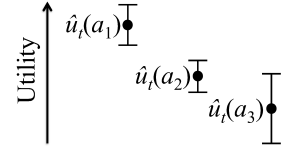}
  \caption{Example with $n=3$ after running SWAP for $t$ steps. Dots are the empirical utility $u_t(a)$ while flags represent the radius of confidence $\rad_t(a)$. Here, $\rad_t(a_2)$ and $\rad_t(a_3)$ overlap; SWAP may pull $a_3$.}
  \label{fig:ex_swap}
\end{figure}

\paragraph*{Analysis.}
We now formally analyze SWAP. We define $\bar{X}_{\Cost}=E[\Cost]$ as the expected cost (or expected $j$ value) and $\bar{X}_{\Gain}=E[\Gain]$ as the expected gain (or the expected $s$ value).
Assume that each arm $a \in [n]$ has mean $u(a)$ with an $\sigma$-sub-Gaussian tail. Following Chen et al., 
set $\rad_t(a)=\sigma\sqrt[]{2\log\left(\frac{4n\Cost_t^3}{\delta}\right)/T_t(a)}$ for all $t>0$.

Notice that if we use strong pull policy $\mathit{spp}(s,j)=0$, then we only perform weak arm pulls, and SWAP reduces to Chen et al.'s CLUCB. We call this reduction the \emph{weak only pull problem}.
Chen et al. proved that CLUCB returns the optimal set $M^*$ and uses at most $\tilde{O}(\width(\mathcal{M})^2\mathbf{H})$ samples.
Similarly, if we set $spp(s,j)=1$ then we only perform strong arm pulls---dubbed the \emph{strong only pull problem}. We show that this version of SWAP returns the optimal set $M^*$ and costs at most $\tilde{O}(\width(\mathcal{M})^2\mathbf{H}/s)$.

\begin{theorem} \label{thm:strong_only}
Given any $\delta \in (0,1)$, any decision class $\mathcal{M}\subseteq 2^{[n]}$, and any expected rewards $\mathbf{u} \in \mathbb{R}^n$, assume that the reward distribution $\varphi_a$ for each arm $a\in[n]$ has mean $u(a)$ with an $\sigma$-sub-Gaussian tail. Let $M^*=\arg\max_{M\in\mathcal{M}}w(M)$ denote the optimal set. Set $\rad_t(a)=\sigma\sqrt[]{2\log\left(\frac{4nt^3j^3}{\delta}\right)/T_t(a)}$ for all $t>0$ and $a\in[n]$. Then, with probability at least $1-\delta$, the SWAP algorithm with only strong pulls where $j\geq1$ and $s>j$ returns the optimal set $\texttt{Out}=M^*$ and
\begin{align} \label{eq:strong_theorembound}
T\leq O\left(\frac{\sigma^2\width(\mathcal{M})^2\mathbf{H}\log(nj^3\sigma^2\mathbf{H}/\delta)}{s}\right)
\end{align}
where $T$ denotes the total cost used by the SWAP algorithm and $\mathbf{H}$ is defined in Eq.\ref{eq:hardness}.
\end{theorem}

Although $s$ and $j$ are problem-specific, it is important to know when to use the strong only pull problem over the weak only pull problem. Corollary \ref{cor:strong_better_than_weak} provides weak bounds for $s$ and $j$ for the strong only pull problem. We also explore its ramifications experimentally in Figure~\ref{fig:cor} as discussed in Section~\ref{sec:experiments-sim}.
\begin{corollary}\label{cor:strong_better_than_weak}
SWAP with only strong pulls is equally or more efficient than SWAP with only weak pulls when $s>0$ and $0<j\leq C^{\frac{s}{3}-\frac{1}{3}}$ where $C=4n\tilde{\mathbf{H}}/\delta$.
\end{corollary}

We now address the general case of SWAP, for any probabilistic strong pull policy parameterized by $s$ and $j$. In Theorem~\ref{thm:swap} we show that SWAP returns $M^*$ in $\tilde{O}\left(\width(\mathcal{M})^2\mathbf{H}/\bar{X}_{\Gain}\right)$ samples.

\begin{theorem}\label{thm:swap}
Given any $\delta_1,\delta_2,\delta_3 \in (0,1)$, any decision class $\mathcal{M}\subseteq 2^{[n]}$, and any expected rewards $\mathbf{u} \in \mathbb{R}^n$, assume that the reward distribution $\varphi_a$ for each arm $a\in [n]$ has mean $u(a)$ with an $\sigma$-sub-Gaussian tail. Let $M^*=\arg\max_{M\in\mathcal{M}}w(M)$ denote the optimal set. Set $\rad_t(a)=\sigma\sqrt[]{2\log\left(\frac{4n\Cost_t^3}{\delta}\right)/T_t(a)}$ for all $t>0$ and $a\in[n]$, set $\epsilon_1=\sigma\sqrt[]{2\log\left(\frac{1}{2}\delta_2/T\right)}$, and set $\epsilon_2=\sigma\sqrt[]{2\log\left(\frac{1}{2}\delta_3/n\right)}$. Then, with probability at least $(1-\delta_1)(1-\delta_2)(1-\delta_3)$, the SWAP algorithm (Algorithm \ref{alg:swap}) returns the optimal set $\texttt{Out}=M^*$ and
\begin{equation}
\label{eq:swap_thm}
T\leq O\left(\frac{\sigma^2\width(\mathcal{M})^2\mathbf{H}\log\left(n\sigma^2\left(\bar{X}_{\Cost}-\epsilon_1\right)^3\mathbf{H}/\delta_1\right)}{\bar{X}_{\Gain}-\epsilon_2}\right),
\end{equation}
where $T$ denotes the total cost used by Algorithm \ref{alg:swap}, and $\mathbf{H}$ is defined in Eq. \ref{eq:hardness}.
\end{theorem}

It is nontrivial to determine where the general version of SWAP is better than both the SWAP algorithm with only strong pulls and the SWAP algorithm with only weak pulls, given the non-asymptotic nature of all three bounds (Chen et al. results and Theorems~\ref{thm:strong_only} and~\ref{thm:swap}).  Based on our experiments (\S\ref{sec:experiments}), we conjecture that there is a of $s$ and $j$ pairs where SWAP is the optimal algorithm, even for relatively low numbers of arm pulls, though it is problem-specific. This is discussed more in Section~\ref{sec:mech-design}.

%% file: experiments.tex
\section{Top-K Experiments}\label{sec:experiments}

In this section, we experimentally validate the SWAP algorithm under a variety of arm pull strategies.
We first explore (\S\ref{sec:experiments-sim}) the efficacy of our bounds in Theorem~\ref{thm:swap} and Corollary~\ref{cor:strong_better_than_weak} in simulation.
Then we deploy SWAP on real data (\S\ref{sec:linear_experiments_real}) drawn from one of the largest computer science graduate programs in the United States. We show that SWAP provides a higher overall utility with equivalent cost to the actual admissions process.

\subsection{Gaussian Arm Experiment}\label{sec:experiments-sim}
We begin by validating the tightness of our theoretical results in a simulation setting that mimics the assumptions made in Section~\ref{sec:alg}.   We pull from a Gaussian distribution around each arm. When arm $a$ is weak pulled, a reward is pulled from a Gaussian distribution with mean $u_a$, the arm's true utility, and standard deviation $\sigma$. Similarly, when arm $a$ is strong pulled, the algorithm is charged $j$ cost, and a reward is pulled from a distribution with mean $u_a$ and standard deviation $\sigma/\sqrt[]{s}$. This strong pull distribution is equivalent to pulling the arm $s$ times and averaging the reward, thus ensuring an information gain of $s$.

We ran all three algorithms---SWAP with the strong pull policy defined in Equation~\ref{eq:spp}, SWAP with only strong pulls, and SWAP with only weak pulls---while varying $s$ and $j$. For each $s$ and $j$ pair we ran the algorithms at least $4,000$ times with a randomly generated set of arm values. Random seeds were maintained across policies. We then compared the cost of running each of the algorithms.\footnote{All code to replicate this experiment can be found here: \url{https://github.com/principledhiring/SWAP}.}

\begin{figure}
\centering
\includegraphics[width=0.65\columnwidth]{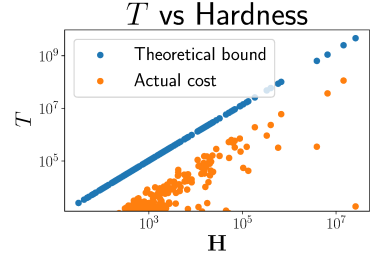}
\vspace{-10pt}
\caption{Exploration of bounds in practice vs. the theoretical bounds of Theorem~\ref{thm:swap} with respect to hardness (note that both axes are a log scale).}
\label{fig:linear_hardness}
\end{figure}

To test Corollary~\ref{cor:strong_better_than_weak}, Figure~\ref{fig:cor} compares SWAP with only weak pulls to SWAP with only strong pulls.  We found that Corollary~\ref{cor:strong_better_than_weak} is a weak bound on the boundary value of $j$.
The general version of SWAP should be used when it performs better---costs less---than both the strong only and weak only versions of SWAP. The zone where SWAP is effective varies with the problem (See \S\ref{sec:mech-design} for a deeper discussion). Figure~\ref{fig:opt_zone_gaussian} shows the optimal zone for the Gaussian Arm Experiment.


\begin{figure}
\centering
  \begin{subfigure}[b]{0.5\columnwidth}
    \includegraphics[width=\textwidth]{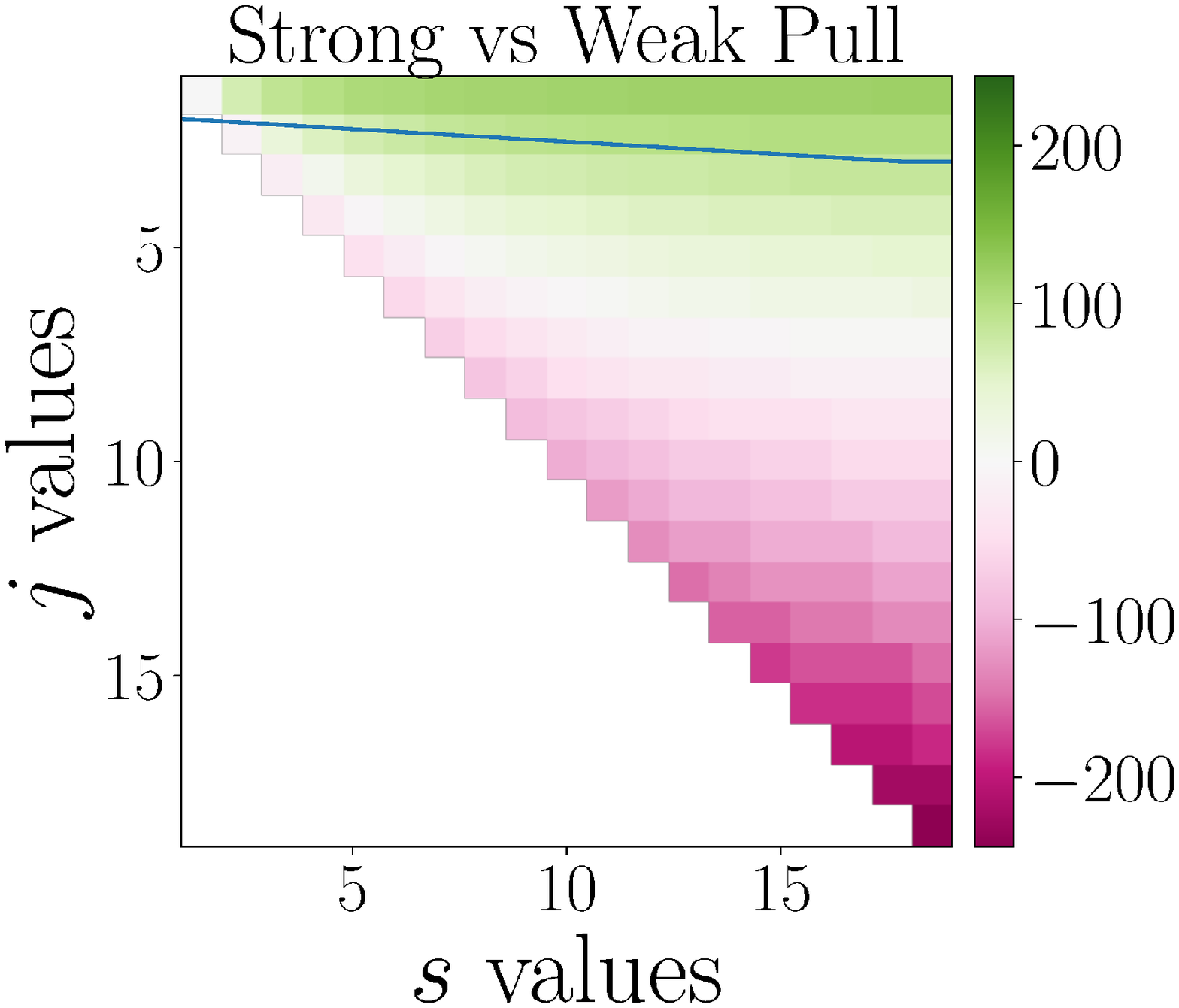}
    \vspace{-20pt}
    \caption{Weak vs Strong}
    \label{fig:cor}
  \end{subfigure}
  \begin{subfigure}[b]{0.45\columnwidth}
    \includegraphics[width=\textwidth]{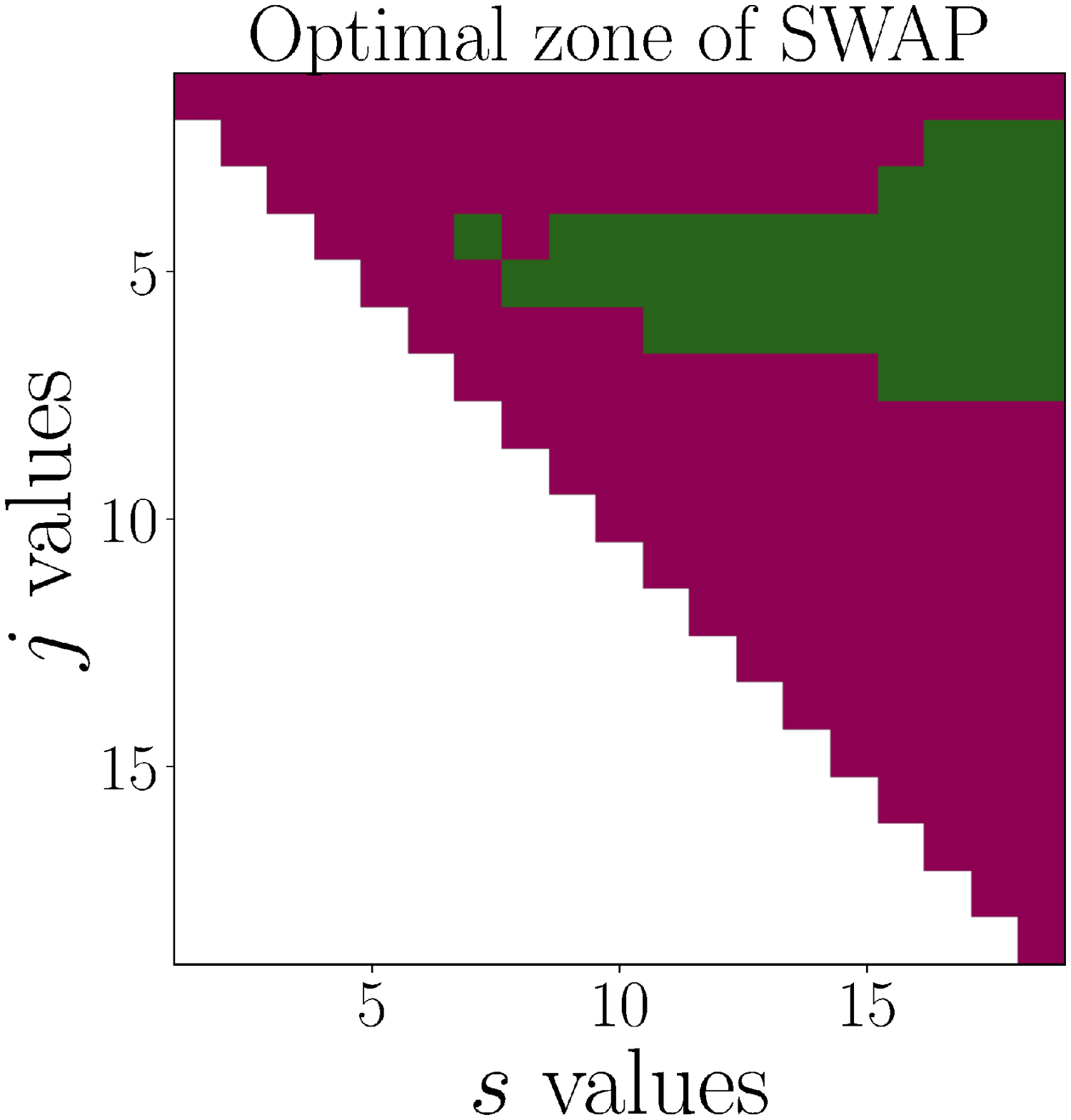}
    \caption{SWAP Optimal Zone}
    \label{fig:opt_zone_gaussian}
  \end{subfigure}
  \caption{Cost comparisons. Figure \ref{fig:cor} compares only strong to only weak pulls. Green indicates better performance by strong pulls, and intensity indicates magnitude. The blue line is the Corollary~\ref{cor:strong_better_than_weak} bound on $j$. Figure \ref{fig:opt_zone_gaussian} shows where the general version of SWAP outperformed (green) both SWAP with only strong pulls as well as SWAP with only weak pulls, and (maroon) where it outperformed at least one of the latter.}
  \label{fig:rand_uni}
  \vspace{-10pt}
\end{figure}

\subsection{Graduate Admissions Experiment}\label{sec:linear_experiments_real}

Finally, we describe a preliminary exploration of SWAP on real graduate admissions data from one of the largest CS graduate programs in the United States. The experiment was approved by the university's Institutional Review Board. Our dataset consists of three years of graduate admissions applications, graduate committee application review text and ratings, and final admissions decisions. Information was gathered from the first two academic years (treated as a training set), while the data from last academic year was used to evaluate the performance of SWAP (treated as a test set).

\paragraph{Dataset.}
During the admissions process, potential students from all over the world send in their applications. A single application consists of quantitative information such as GPA, GRE scores, TOEFL scores, nationality, gender, previous degrees and so on, as well as qualitative information in the form of recommendation letters and statements of purpose.  In the 2016-17 academic year, the department received approximately 1,600 applications, with roughly 4,500 applications over all three years. The most recent 1,600 applications are roughly split into 1,000 Master's applications and 600 Ph.D.\ applications. The acceptance rate is 3\% for Masters students  and 20\% for Ph.D. students.

Once all applications are submitted, they are sent to a review committee. 
Generally, applicants at the top (who far exceed expectations) and applicants at the bottom (who do not fulfill the program's strict requirements) only need one review. Applicants on the boundary, however, may go through multiple reviews with different committee members. Once all reviews have been made, the graduate chair chooses the final applicants to admit.

By administering an anonymous survey of past admissions committee members, we estimated that interviews are approximately six times longer than reviewing a written application. Therefore, we set our $j$ value (the cost of a strong pull) to be $6$. The gain of an interview is uncertain, so we ran tests over a wide range of $s$ values (the information gain of a strong pull). The number of reviews and interviews ($\times 6$) were summed to get a cost T of the actual review process.

\paragraph{Experimental Setup.}
We simulate an arm pull by returning a real score that a reviewer gave during the admissions process (in the order of the original reviews) or a score from a probabilistic classifier (if all committee members' reviews have been used). An arm pull returns a score drawn from a distribution around the probabilistic result from the classifier to simulate some human error or bias. 

We ran SWAP using the strong pull policy defined in Eq.~\ref{eq:spp}, where we define the utility of each arm by the probabilistic result from the classifier. For our results, we compare SWAP's selections with the real decisions made during the admissions process. 

\paragraph{Results.}

Running SWAP consistently resulted in a higher overall utility than the actual admissions process while using roughly equivalent cost (Table~\ref{table:admissions_linear}). We see that the overall top-K utility $w$ is higher in SWAP than in practice. We also see that SWAP uses roughly equivalent resources $T$ than what is used in practice. This suggests that SWAP is a viable option for admissions. There are, however, some limitations of only using a top-K policy, such as potentially overlooking the value diverse candidates bring to a cohort. For instance, when hiring a software engineering team, if the top candidates are all back-end developers, it may be worthwhile to hire a front-end developer with slightly lower utility.

\begin{table}
    \centering
    \begin{tabular}{ l | r  r} 
      & $w$ & $T$ \\
    \midrule
    SWAP & 80.1 (0.5) & 1978 (53)  \\ 
    Actual & 73.96 & \textasciitilde 2000
  \end{tabular}
  \caption{Graduate Admissions Simulation of SWAP. Comparison of top-K utility $w$ and cost $T$ of SWAP with results of the actual admissions process. The values in parentheses are the standard deviations.}
  \label{table:admissions_linear}
  \vspace{-10pt}
\end{table}

\section{Promoting diversity through a submodular function}
Motivated by recent evidence that diversity in the workforce can increase productivity~\cite{Hunt15:Diversity,Desrochers01:Local}, 
we explore the effect of formally promoting diversity in the cohort selection problem.  
First, we define a submodular function that promotes diversity (Section~\ref{sec:diversity_function}).
Then empirically, we show that SWAP performs well with a submodular objective function (Section~\ref{sec:diverse_gaussian}).  In experiments on real data, we show a significant increase in diversity with little loss in fit while using roughly the same 
resources as in practice (Section~\ref{sec:submod-experiments-real}). 

\subsection{Diversity Function}\label{sec:diversity_function}

Quantifying the diversity of a set of elements is of interest to a variety of fields, including recommender systems, information retrieval, computer vision, and others~\cite{Qin13:Promoting,ashkan2015optimal,Sha16:Framework,radlinski2008learning}.
For our experiments, we choose a recent formalization from \citet{lin2011class} and apply it to both simulated and real data. 
Their formulation assumes that the arms can be split into $L$ partitions where a partition is denoted as $P_i$ and a cohort is defined as $M=P_1\cup P_2 \cup\ldots\cup P_L$.
At a high level, the diversity function $\wdiverse{}$ is defined as 
$\wdiverse(M) = \sum_{i=1}^L \sqrt[]{\sum_{a \in P_i} u(a)}$.
Lin and Bilmes showed that $\wdiverse{}$ is submodular and monotone.  Under $\wdiverse(M)$ there is typically more benefit to selecting an arm from a class that is not already represented in the cohort, if the empirical utility of an arm is not substantially low. As soon as an arm is selected from a class, other arms from that class experience diminishing gain due to the square root function.  Example~\ref{ex:diversity-vs-fit} illustrates when $\wdiverse{}$ results in a different cohort selection than the top-K function $\wtop{}(M) = \sum_{a \in M} u(a)$.

\begin{example}\label{ex:diversity-vs-fit}
Return to a similar setting to Example~\ref{ex:swap}, with three arms $\{a_1, a_2, a_3\} = A$ and true utilities $u(a_1) = 0.6$, $u(a_2) = 0.5$, and $u(a_3) = 0.3$.  Assume there exist $L=2$ classes, and let arms $a_1$ and $a_2$ belong to class $1$, and arm $a_3$ belong to class $2$.  Then, for a cohort of size $K=2$, $\wtop{}$ will select cohort $\Mtop{} = \{a_1,a_2\}$, while $\wdiverse{}$ will select cohort $\Mdiverse{} = \{a_1,a_3\}$.  Indeed, $\wtop{}(\Mtop{}) = 1.1 > 0.9 = \wtop{}(\Mdiverse{})$, while $\wdiverse{}(\Mtop{}) = \sqrt{1.1} \approx 1.05 < 1.3 \approx \sqrt{0.6}+\sqrt{0.3} = \wdiverse{}(\Mdiverse{})$.
\end{example}

Maximizing a general submodular function is computationally difficult.  \citet{nemhauser1978analysis} proved that a close to optimal---that is, $\wdiverse{}(M^*) \geq \left(1-\frac{1}{e}\right) \texttt{OPT}$---greedy algorithm exists for submodular, monotone functions that are subject to a cardinality constraint.  We use that standard greedy packing algorithm in our implementation of the oracle.

\subsection{Diverse Gaussian Arm Experiments}\label{sec:diverse_gaussian}
To determine if SWAP works in this submodular setting, we ran simulations over a variety of hardness levels. We instantiated the problem similarly to that of Section~\ref{sec:experiments-sim} with the added complexity of dividing the arms into three partitions. 

Figure~\ref{fig:submod_hardness} shows the cost of running SWAP compared to the theoretical bounds of the linear model over increasing hardness levels. The results show that SWAP performs well for the majority of cases. However, for some cases, the cost becomes very large. To deal with those situations, we can use a probably approximately correct (PAC) relaxation of Algorithm~\ref{alg:swap} where Line~\ref{line:equality} becomes $\texttt{If } \left|w(\tilde{M}_t) - w(M_t)\right| \leq \epsilon$. The results from this PAC relaxation where $\epsilon=0.01$ can be found in Figure~\ref{fig:pac_hardness}. Note that the definition of hardness found in Equation~\ref{eq:hardness} does not quite fit this situation since the graphs in Figure~\ref{fig:submod_hardness_figs} have higher costs for some lower hardness problems while having lower cost for some higher hardness problems. Given that the PAC relaxation performs well with low costs over all of the tested hardness problems, we propose that SWAP can be used with $\wdiverse$ and perhaps other submodular and monotone functions.

\begin{figure}
\centering
\begin{subfigure}[b]{0.49\columnwidth}
\includegraphics[width=\textwidth]{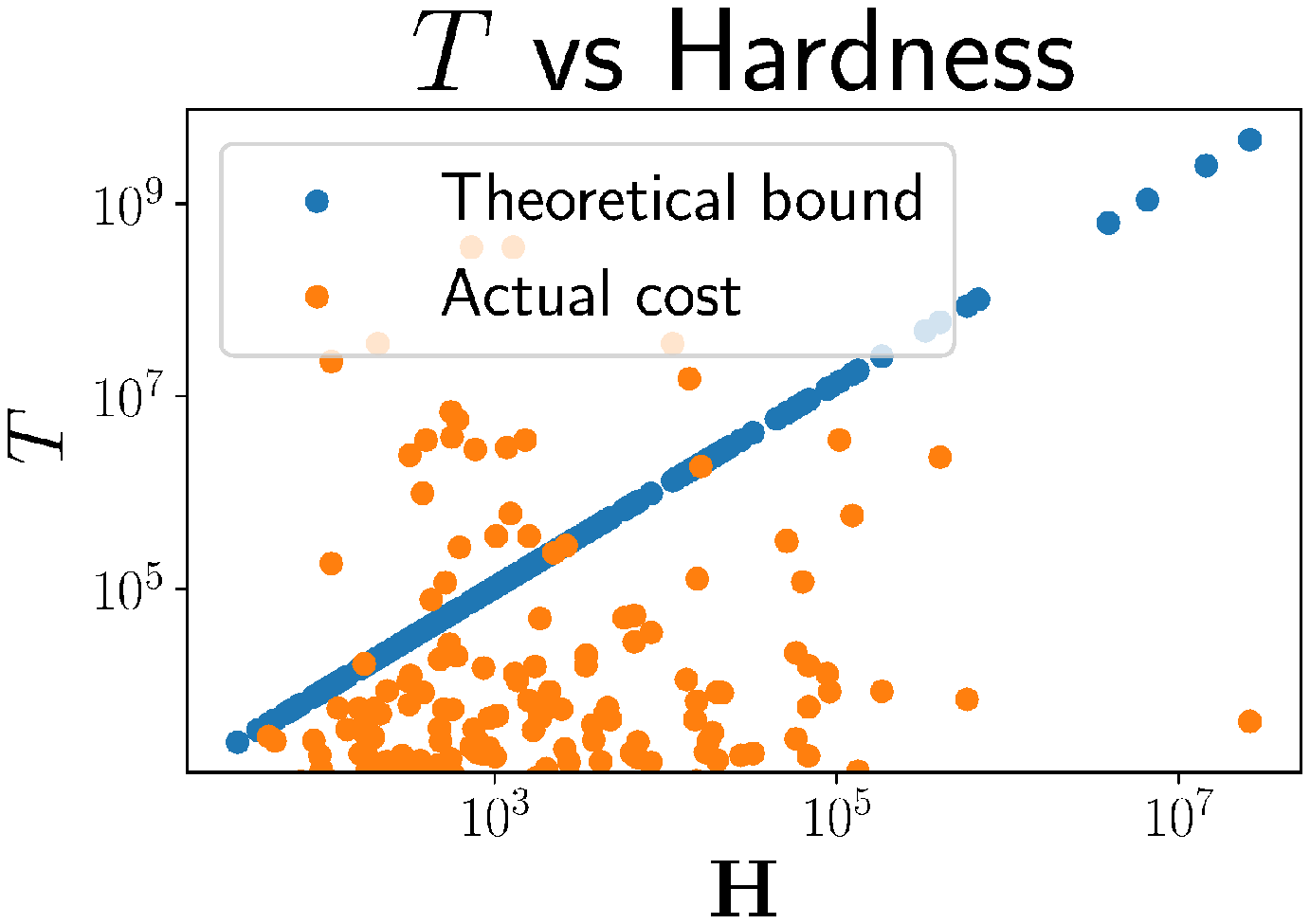}
    \caption{SWAP with $\wdiverse$}
    \label{fig:submod_hardness}
\end{subfigure}
\begin{subfigure}[b]{0.49\columnwidth}
\includegraphics[width=\textwidth]{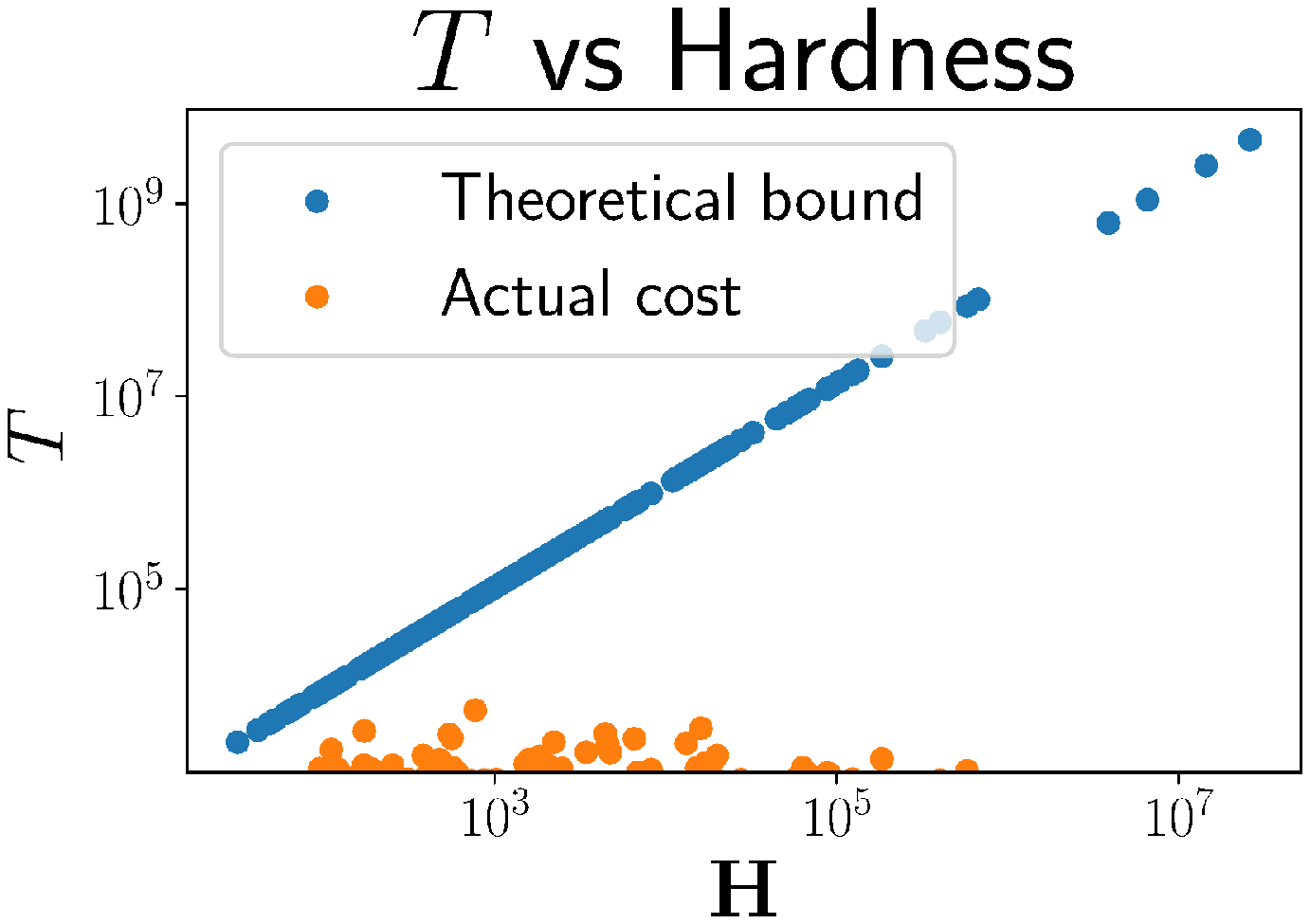}
    \caption{PAC relaxation with $\wdiverse$}
    \label{fig:pac_hardness}
\end{subfigure}
\caption{Exploration of bounds in practice for SWAP with $\wdiverse$ (\ref{fig:submod_hardness}) and the PAC relaxation of SWAP with $\wdiverse$ (\ref{fig:pac_hardness}) vs. the theoretical bounds of Theorem~\ref{thm:swap} with respect to hardness (Note that both axes are a log scale).}
\label{fig:submod_hardness_figs}
\vspace{-10pt}
\end{figure}



\subsection{Diverse Graduate Admissions Experiment}\label{sec:submod-experiments-real}
Using the same setting as described in Section~\ref{sec:linear_experiments_real}, we simulate a SWAP admissions process with the submodular function $\wdiverse{}$. We partition groups by gender (which is binary in our dataset) and multi-class region of origin. We found that we did not have to resort to the PAC version of SWAP to tractably run the simulation  over various partitions of the graduate admissions data.


\begin{figure}[ht]
\centering
  \begin{subfigure}[b]{0.43\columnwidth}
    \includegraphics[width=\textwidth]{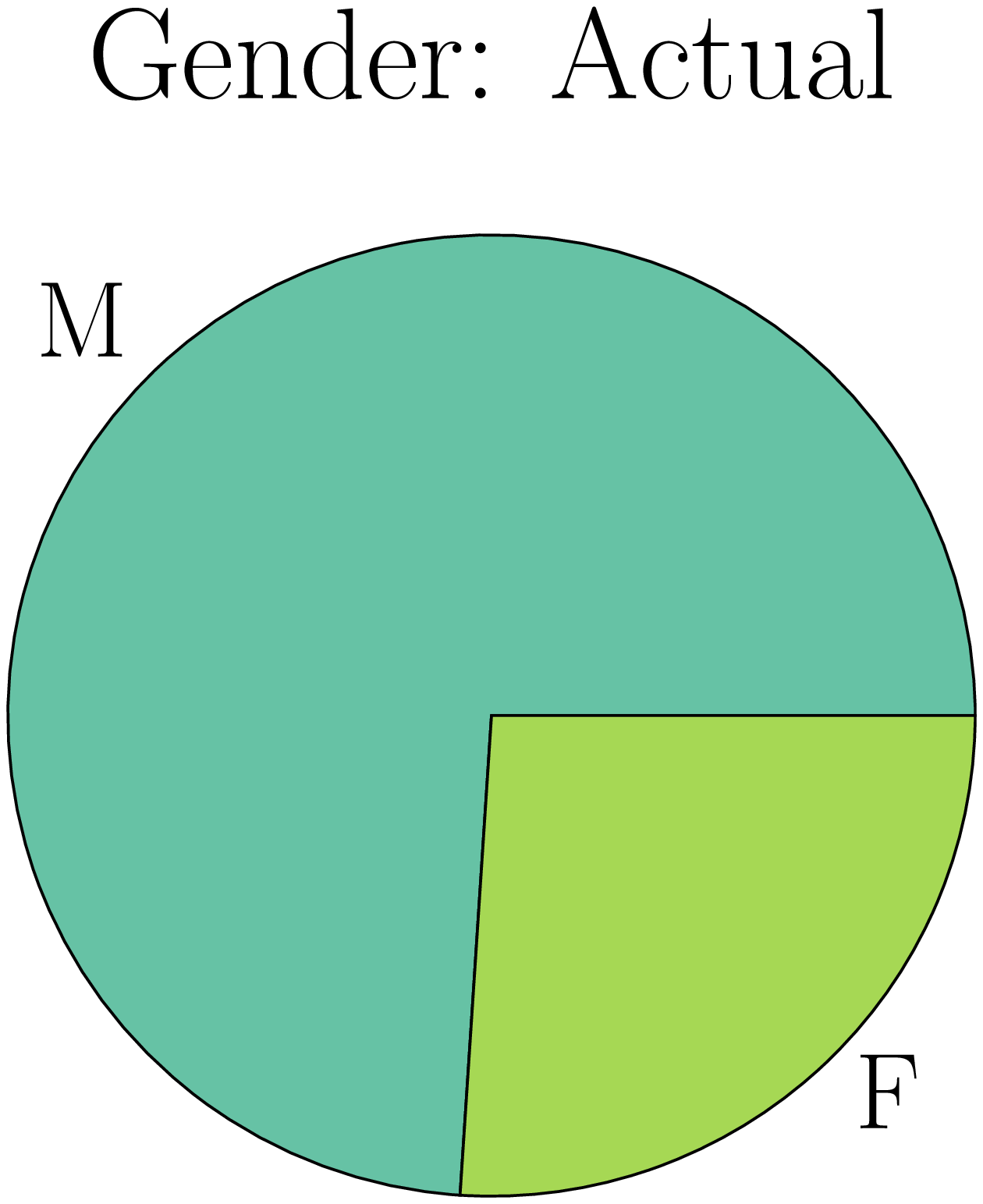}
    \caption{Actual}
    \label{fig:actual_sex}
  \end{subfigure}
  \begin{subfigure}[b]{0.43\columnwidth}
    \includegraphics[width=\textwidth, height=20cm,keepaspectratio]{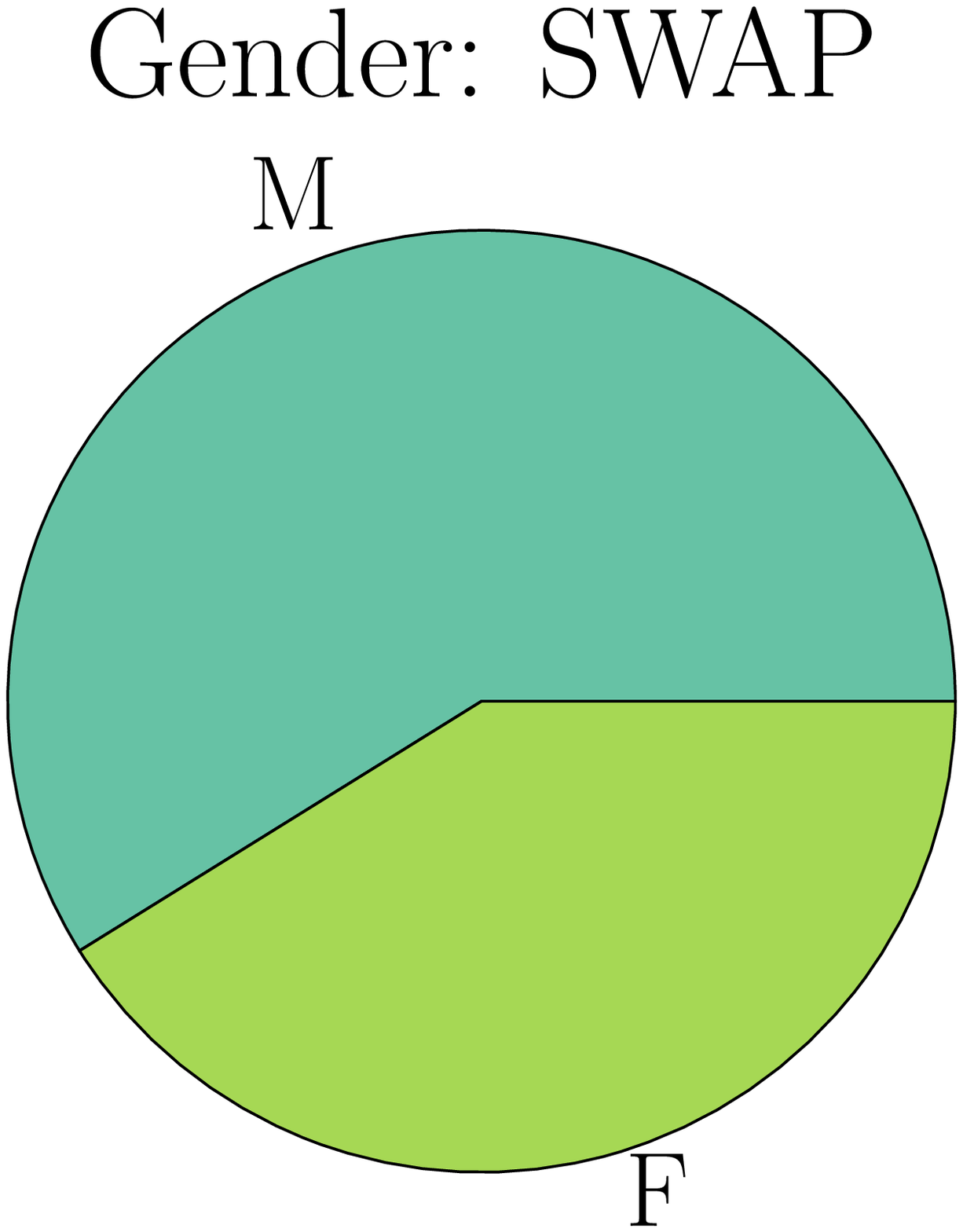}
    \caption{SWAP}
    \label{fig:swap_sex}
  \end{subfigure}
  \begin{subfigure}[b]{0.43\columnwidth}
    \includegraphics[width=\textwidth]{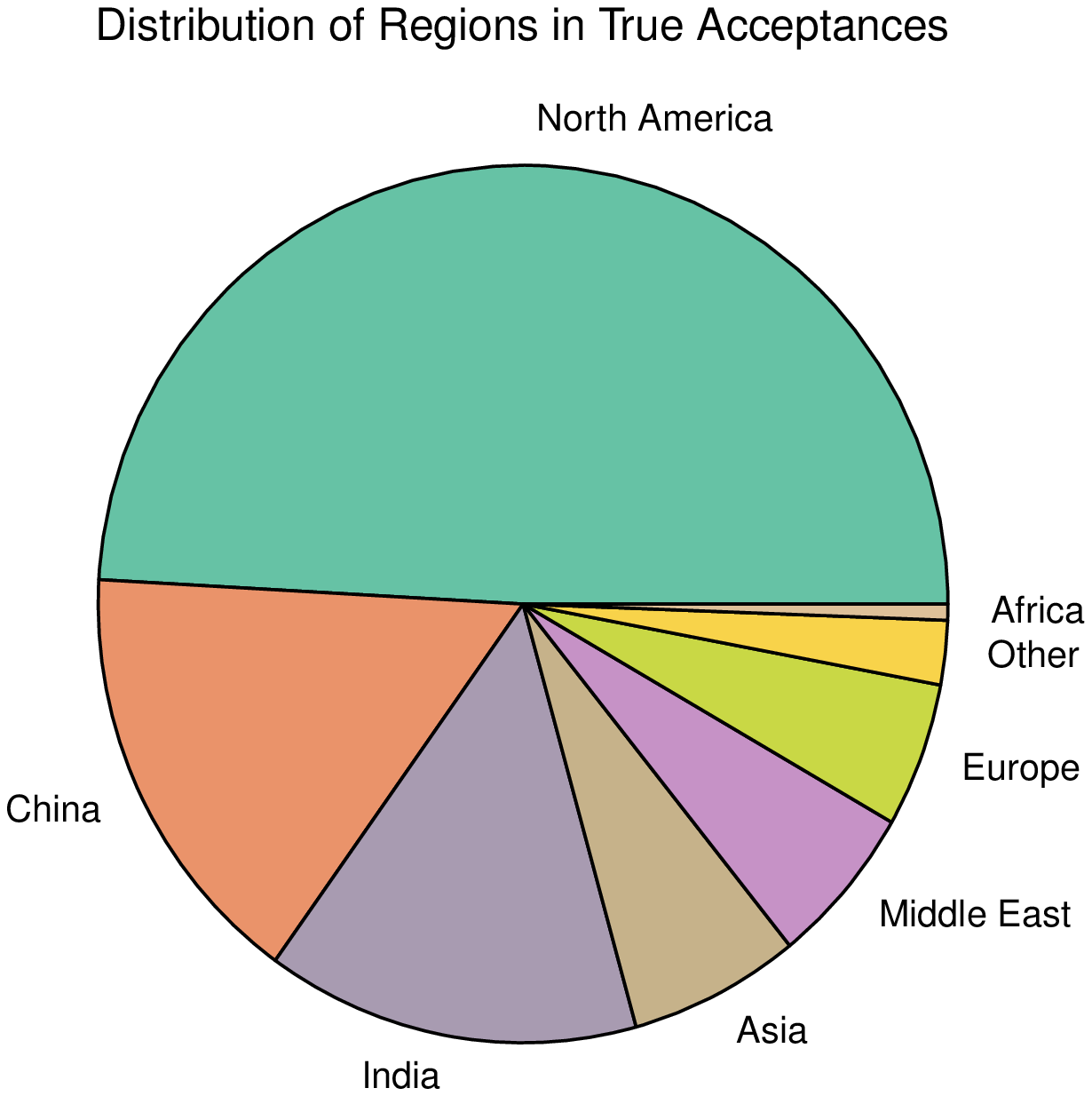}
    \caption{Actual}
    \label{fig:actual_region}
  \end{subfigure}
  \begin{subfigure}[b]{0.43\columnwidth}
    \includegraphics[width=\textwidth]{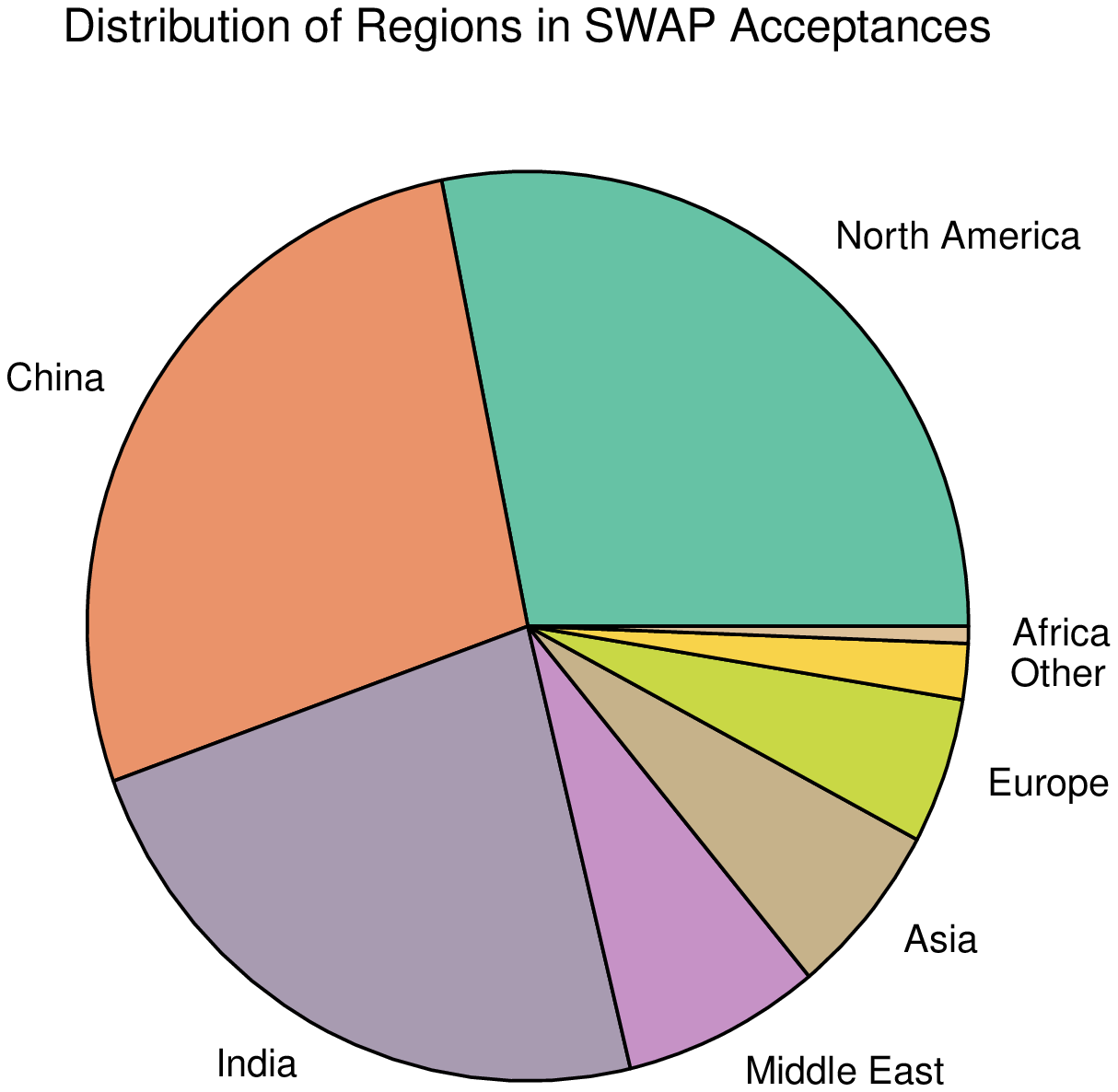}
    \caption{SWAP}
    \label{fig:swap_region}
  \end{subfigure}

  \caption{Comparison of true and SWAP-simulated admissions: gender (\ref{fig:actual_sex}, \ref{fig:swap_sex}) \& region (\ref{fig:actual_region}), \ref{fig:swap_region}).}
  \label{fig:sex}
  \vspace{-10pt}
\end{figure}


\paragraph{Results.}
\begin{table}
    \centering

    \begin{tabular}{ l | r  r | r  r } 
       & \multicolumn{2}{c|}{Gender} & \multicolumn{2}{c}{Region of Origin}  \\[0.5ex] 
      & $\sqrt{\wtop{}}$ & $\wdiverse{}$ & $\sqrt{\wtop{}}$ & $\wdiverse{}$ \\ 
    \midrule
    SWAP & 8.5 (0.03)& 12.1 (0.06) & 8.0 (0.03)& 22.1 (0.03)  \\ 
    Actual & 8.6 & 11.8 & 8.6 & 20.47
  \end{tabular}
  \caption{SWAP's average gain in diversity over different classes.}
  \label{table:diversity}
\end{table}
We compare two objective functions, $\wtop{}$ and $\wdiverse{}$. $\wtop{}$ treats all applicants as members of one global class. This mimics a top-K objective, where applicants are valued based on individual merit alone. 
$\wdiverse{}$ promotes diversity using reported gender  and region of origin for class memberships. We use those classes as our objective during separate runs of SWAP. 

\begin{figure}
\centering
  \begin{subfigure}[b]{0.55\columnwidth}
    \includegraphics[width=\textwidth]{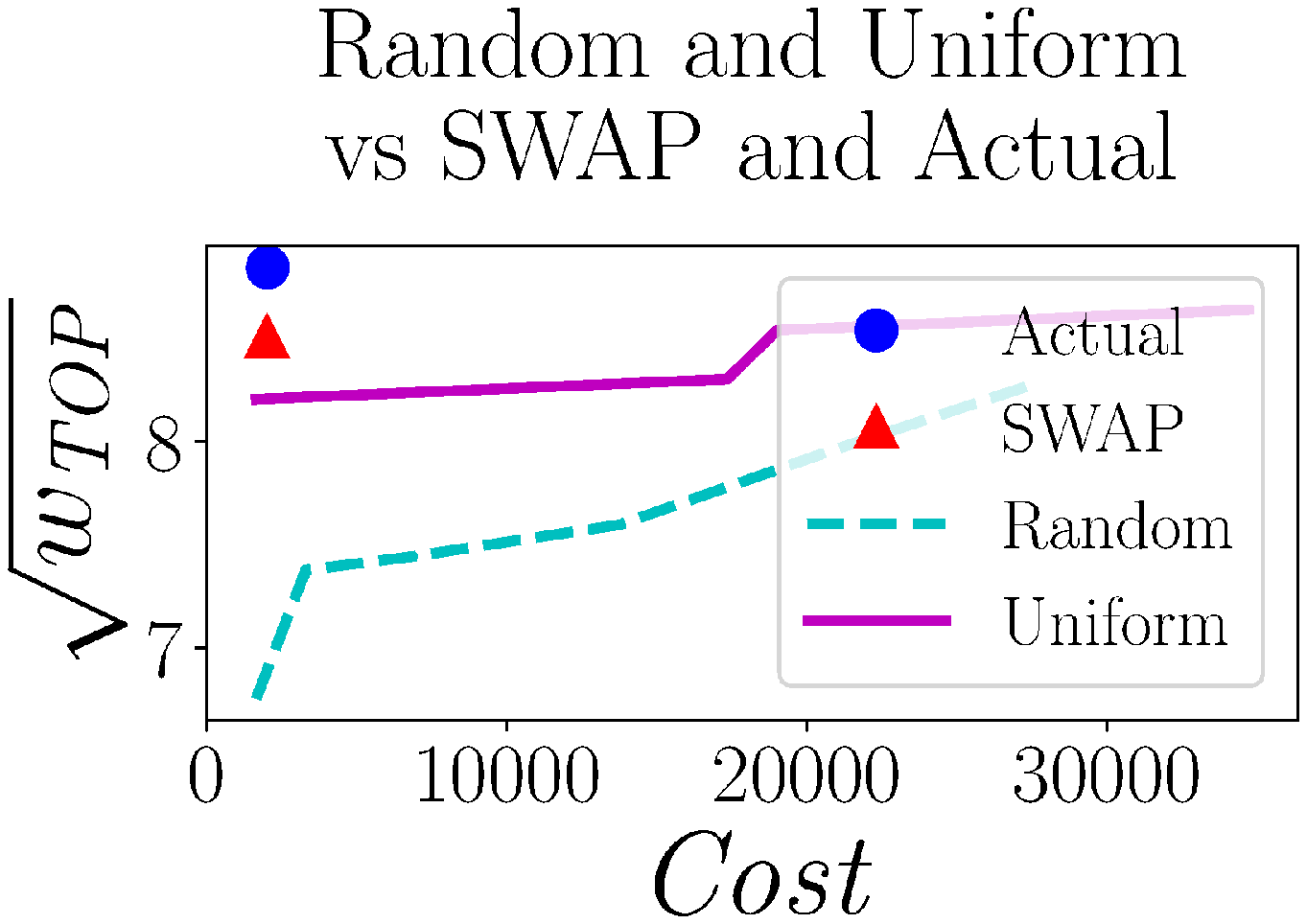}
    \caption{$\sqrt[]{\wtop{}}$}
    \label{fig:rand_uni_top}
  \end{subfigure}
  \begin{subfigure}[b]{0.55\columnwidth}
    \includegraphics[width=\textwidth]{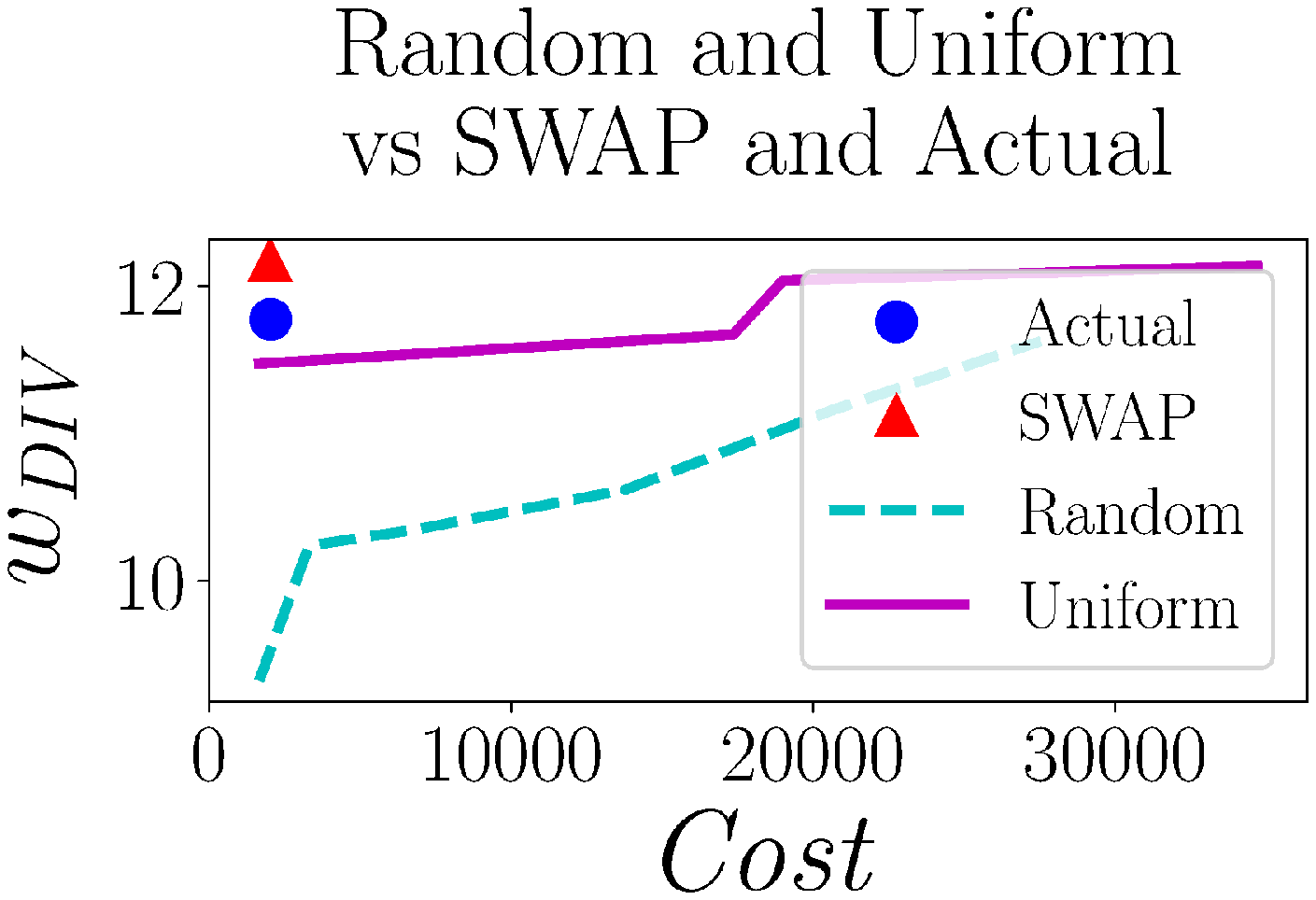}
    \caption{$\wdiverse{}$}
    \label{fig:rand_uni_div}
  \end{subfigure}
  \caption{Cost vs utility function comparisons of Actual, SWAP, Random, and Uniform.}
  \label{fig:rand_uni}
\end{figure}

Table~\ref{table:diversity} and Figure~\ref{fig:sex} show experimental results on the test set (most recent year) of real admissions data. We report $\sqrt{\wtop{}}$ instead of $\wtop{}$ to align units across objective functions. Because the square root function is monotonic, this conversion does not impact the maximum utility cohort. Since SWAP uses a diversity oracle (\S\ref{sec:diversity_function}), we notice a slight drop in top-K utility. However, there is a large gain in diversity.




SWAP, on average, used 1.17 pulls per arm, of which 5\% were strong. During the last admissions decision process each applicant was reviewed on average 1.21 times. Interviews were not consistently documented. SWAP performed more strong pulls (interviews) of applicants than our estimation of interviews by the graduate admissions committee, but did fewer weak pulls. SWAP spent roughly the same amount of total resources as the committee did with strong pull cost $j=6$ and weak pull cost of $1$.  Given the gains in diversity, this supports SWAP's potential use in practice.

We also compare SWAP to both uniform and random pulling strategies, shown in Table~\ref{fig:rand_uni}. The uniform strategy weak pulls each arm once and strong pulls each arm once. This had a cost approximately 9 times that of SWAP and resulted in a general utility of 8.3 and a diversity value of 11.8. The random strategy weak or strong pulls arms randomly. Even when spending 10 times the cost of running SWAP, the random strategy has only a general utility of 7.9 and a diversity value of 11.16. SWAP significantly outperforms both of these strategies.



%% file: discussion.tex
\section{Discussion}
Admissions and hiring are extremely important processes that affect individuals in very real ways. Lack of structure and systematic bias in these processes, present in application materials or in resource allocation, can negatively affect applicants from traditionally underrepresented minority groups. We suggest a formally structured process to help prevent disadvantaged people from falling through the cracks. We discuss benefits (Section~\ref{sec:benefits}) and limitations (Section~\ref{sec:limitations}) to this approach, as well as mechanism design suggestions for deploying SWAP in practice (Section~\ref{sec:mech-design}).

\subsection{Benefits}\label{sec:benefits}
We established SWAP, a clear-cut way to model a sequential decision-making process where the aim is to select a subset using two kinds of information-gathering strategies as a multi-armed bandit algorithm. This process could have a number of benefits when used in practical hiring/admissions settings.

Over the course of designing and running our experiments, we noticed what seemed like bias in the application materials of candidates belonging to underrepresented minority groups. Our initial observations were similar to those  of scholars such as \citet{schmader2007}, who found that recommendation letters for female applicants to faculty jobs contained fewer work-specific terms than male applicants. After revisiting and coding application materials in our experiments, we found similar results for female and other minority candidates. 

Our process hopes to mitigate this bias by providing a completely structured process, informed by the many studies showing that structured interviewing reduces bias (see Section \ref{relatedwork}). 
As we showed in our experiments, one can take additional steps to encourage diversity (by using $\wdiverse{}$) to select a more diverse team, which can result in a less biased, more productive work environment~\cite{Hunt15:Diversity}.

Furthermore, by including a diversity measure in the objective function, candidates from disadvantaged groups are given a higher chance of being pulled through the cracks since we prioritize recommending diverse candidates for additional resource allocation. 

A practical benefit to SWAP is that it avoids spending unnecessary resources on outlier candidates and quickly finds uncertain candidates. This give us more information about the applicant pool as whole, allowing us to make better decisions when choosing a cohort while using roughly equivalent resources.

Finally, in our simulations of running SWAP during the graduate admissions process, we also select a more diverse student cohort at low cost to cohort utility.

\subsection{Limitations}\label{sec:limitations}
One significant limitation of a large-scale system like SWAP is that it relies on having a utility score for each applicant. 
In our graduate admissions experiment, we assume the true utility of an applicant can be modeled by our classifier, which is not entirely accurate. In reality, the true utility of an applicant is nontrivial to estimate as it is subjective and depends on a wide range of factors. Finding an applicant's true utility would require following and evaluating the applicant through the end of the program, perhaps even after they have left the university. Even if that were possible, being able to quantify true utility is nontrivial due to the subjectivity of success and its qualitative properties.
This problem is not limited to SWAP--it is present in any admissions, hiring, peer review, and other processes that attempt to quantify the value of qualitative properties. Therefore in these settings there is no choice but to rely on proxy values for the true utility, such as reviewer scores.

Similarly, even though the cost of a resource, $j$, may be inherently quantifiable, the information gain $s$, is harder to define in such a process. For example, how much more information one gains from an interview over a resume review is subjective and, by nature, more qualitative than quantitative. Also, the information gain from expending the same resource may vary over applicants, though this is slightly mitigated by using structured interviews.

Another limiting factor is that not every admitted applicant will matriculate into the program. We assume that all applicants will accept our offer, but in reality, that is not the case. Therefore, we potentially reject applicants that would matriculate, as opposed to accepting higher quality applicants that will ultimately not. 

Finally, our graduate admissions experiment \emph{simulated} strong arm pulls: reviewers did not give  additional interviews of applicants during the experiment.  Although our results are promising, SWAP should be run in conjunction with an actual admissions process to assess its true performance.

\subsection{Design Choices}\label{sec:mech-design}
Our motivation in designing SWAP and exploring related extensions is to aid hiring and admissions processes that use structured interviewing practices and aim to hire a diverse cohort of workers. As with any algorithm deployed in practice, actually \emph{running} SWAP alongside a hiring process requires adaptation to the specific environment in which it will be used (e.g., batch versus sequential review), as well as estimation of parameters involving correctness guarantees (e.g., $\delta$ and $\epsilon$) or population estimates (e.g., $\sigma$).

In general, we recommend that the policymaker or mechanism designer tasked with setting parameters for SWAP, or a SWAP-style algorithm, should conduct a study on past admissions/hiring decisions. This study should include quantitative information (e.g., how many people applied, how many were accepted, how many were interviewed, how long did interviews take) and qualitative information (e.g., how confident was reviewer A after reviewing an applicant B). From this a mechanism designer could determine estimates of population parameters like $\sigma$, information gain parameters $s$, and interview cost parameter $j$. 

To estimate $\sigma$, a policymaker could perform a study on past reviews and interviews to determine the range of scores for arms. 
However, this method could incorporate various biases that may already exist in prior review and scoring processes.  That consideration should be taken into account, but exactly how is situation-specific. The introduction of and strict adherence to the structured interview paradigm is a general method to alleviate some of these concerns.

To estimate the value of $s$, the information gain of a strong pull, one could quantify the difference in confidence level for a particular applicant after performing weak and strong pulls; e.g., how confident was reviewer $A$ after reviewing an applicant $B$, how much more confident was $A$ after interviewing $B$, and so on. For $j$, policy makers could use the average relative difference in time (and possibly monetary) resources spent on different information gathering strategies.

The choice of $\delta$ and $\epsilon$ could be determined via a sensitivity-analysis-style study, where simulations are run using various settings of $\delta$ and $\epsilon$. Policymakers can then judge the simulated risks and rewards to define the parameters.

Once the hyper-parameters have been found, simulations can be performed to find the optimal zone (as discussed in Section~\ref{sec:experiments-sim}). This will allow the designer to determine the best strong pull policy.

Ideally, both studies should include a run focused on past decisions and one run every time the selection process occurs, to ensure SWAP's parameters align with the experiences and values of human decision-makers.

%% file: conclusion.tex
\section{Conclusion}\label{sec:conclusions}
In this paper, we modeled the allocation of interviewing resources and subsequent selection of a cohort of applicants as a combinatorial pure exploration (CPE) problem in the multi-armed bandit setting.  We generalized  a recent CPE algorithm to the setting where arm pulls can have different costs--where a decision maker can perform \emph{strong} and \emph{weak} pulls, with the former costing more than the latter, but also resulting in a less noisy signal.  We presented the strong-weak arm-pulls (SWAP) algorithm and proved theoretical upper bounds for a general class of arm pulling strategies in that setting.  We also provided simulation results to test the tightness of these bounds.  We then applied SWAP to a real-world problem with combinatorial structure: incorporating diversity into university admissions. On real admissions data from one of the largest US-based computer science graduate programs, we showed that SWAP produces more diverse student cohorts at low cost to student quality while spending a budget comparable to that of the current admissions process.

It would be of both practical and theoretical interest to tighten the upper bounds on convergence for SWAP, either for a reduced or general set of arm pulling strategies. We would also like to extend SWAP to include more than two types of pulls or information gathering strategies. We aim to incorporate a more realistic version of diversity and achieve a provably \emph{fair} multi-armed bandit algorithm, as formulated by~\citet{FairnessNIPS2016} and~\citet{Liu17:Calibrated}. Additionally, we aim to create a version of SWAP that incorporates applicant matriculation into the candidate-recommending and selection process. 

An interesting direction that may be worth pursuing is drawing connections between our work---the selection of a diverse subset of arms---to recent work in \emph{multi-winner voting}~\cite{Faliszewski17:Multiwinner}, a setting in social choice where a subset of alternatives are selected instead of a single winner.  Recent work in that space looks at selecting a ``diverse but good'' committee of alternatives via social choice methods~\cite{Bredereck18:Multiwinner,Aziz18:Rule}.  Similarly, drawing connections to diversity in allocation and matching problems~\cite{Ahmed17:Diverse,Lian18:Conference,Benabbou18:Diversity} is also potentially of interest.

%% file: appendix.tex
\appendix

\setcounter{theorem}{1}

\section{Table of Symbols}\label{app:variables}
For ease of exposition and quick reference, Table~\ref{tbl:variables} lists each symbol used in the main paper, along with a brief description of that symbol.  (We note that each symbol is also defined in the body of the paper prior to its first use.)

\begin{table}[h!]
\centering
\begin{tabular}{ l | p{6cm}  }
  Variable & Summary \\ \hline
  $n$ & Number of applications \\ \hline
  $K$ & Size of cohort wanted \\ \hline
  $A$ & Set of applications \\ \hline
  $a_i$ & a single application with $i\in[n]$ \\ \hline
  $u(a_i)$ & True utility of arm $a_i$ where $u(a_i)\in[0,1]$ \\ \hline
  $\mathbf{u}$ & The set of true utilities. \\ \hline
  $\hat{u}(a_i)$ & Empirical estimate of utility of arm $a_i$ \\ \hline
  $\rad(a_i)$ & Uncertainty bound around arm $a_i$. The true utility $u(a_i)$ should lie with $\hat{u}(a_i)-\rad(a_i)$ and $\hat{u}(a_i)+\rad(a_i)$ \\ \hline
  $\mathcal{M}$ & Decision class. Set of potential cohorts (subsets of arms). \\ \hline
  $w$ & Submodular and monotone function for total utility of a cohort. $w:\mathcal{M}\times\mathbb{R}^n\rightarrow\mathbb{R}$ \\ \hline
  $\doracle(\cdot)$ & Maximization oracle  \\ \hline
  $M^*$ & The optimal cohort given the true utilities $\mathbf{u}$ and total utility function $w$ \\ \hline
  $\Delta_a$ & Gap score for an arm $a$ defined in Equation \ref{eq:gap} \\ \hline
  $\mathbf{H}$ & Hardness of a problem defined in Equation \ref{eq:hardness} \\ \hline
  $\width(\mathcal{M})$ & The smallest distance between any two sets in $\mathcal{M}$ \\ \hline
  $j$ & Cost of a strong arm pull \\ \hline
  $s$ & Information gain of a strong arm pull (ie. the reward is counted $s$ times and is pulled from a tighter distribution around the true utility of an arm) \\ \hline
  $\Cost_t$ & Total cost of pulling arms up until time $t$ \\ \hline
  $T_t(a)$ & Total information gain for arm $a$ up until time $t$ \\ \hline
  $M_t$ & Best cohort of arms at time $t$, given the empirical utilities \\ \hline
  $\tilde{u}_t(a)$ & Worst case empirical utility of arm $a$ (See lines 9-10 of Algorithm \ref{alg:swap}) \\ \hline
  $\tilde{M}_t$ & Best cohort of arms at time $t$, given worst case empirical utilities \\ \hline
  $\spp(s,j)$ & Strong pull policy probability function. See Equation \ref{eq:spp} for an example \\ \hline
  $\sigma$ & We assume that each arm has a $\sigma$-sub-Gaussian tail \\ \hline
  $\bar{X}_{\Cost}$ & Expected cost (expected $j$ value) \\ \hline
  $\bar{X}_{\Gain}$ & Expected information gain (expected $s$ value) \\ \hline
  $\delta$ & Probability that the algorithms output the best sets (See Theorem \ref{thm:strong_only} and Theorem \ref{thm:swap}) \\ \hline
  $\wdiverse$ & Diversity function \\ \hline
  $\wtop$ & Top-K function. $\sqrt{\wtop}$ is the square-root of the top-K function. \\ \hline
\end{tabular}
\caption{All symbols used in the main paper.}\label{tbl:variables}
\end{table}

\section{CLUCB Algorithm}\label{app:clucb}

The Combinatorial Lower-Upper Confidence Bound (CLUCB) algorithm by \citet{chen2014combinatorial} is shown in Algorithm \ref{alg:clucb}. At the beginning of the algorithm, pull each arm once and initialize the empirical means with the rewards from that first arm pull. During iteration $t$ of the algorithm, first find the set $M_t$ using the Oracle. Then, compute the confidence radius for each arm. Find the worst case for each arm and compute a new set $\tilde{M}_t$ using the worst case estimates of the arms. If the utility of the initial set $M_t$ and the worst case set $\tilde{M}_t$ are equal then output set $M_t$. Pull the most uncertain arm (the arm with the widest radius) from the symmetric difference of the two sets $M_t$ and $\tilde{M}_t$. Update the empirical means.

\begin{algorithm}
\caption{Combinatorial Lower-Upper Confidence Bound (CLUCB)}\label{alg:clucb}
\begin{algorithmic}[1]
\Require Confidence $\delta\ \in (0,1)$; Maximization oracle: $\doracle(\cdot):\mathbb{R}^n \rightarrow \mathcal{M}$
\State Weak pull each arm $a \in [n]$ once.
\State Initialize empirical means $\bar{\mathbf{u}}_n$
\State $\forall a \in [n]$ set $T_n(a)\gets 1$
\For{$t=n,n+1,\ldots$}
	\State $M_t \gets \doracle(\bar{\mathbf{u}}_t)$
	\State $\forall a \in [n]$ compute confidence radius $\rad_t(a)$
	\For{$a = 1,\ldots,n$}
    	\If {$a \in M_t$} $\tilde{u_t}(a) \gets \bar{u}_t(a) - \rad_t(a)$
        \Else {} $\tilde{u_t}(a) \gets \bar{u}_t(a) + \rad_t(a)$
        \EndIf
	\EndFor
    \State $\tilde{M}_t \gets \doracle(\tilde{\mathbf{u}}_t)$
    \If {$\tilde{w}(\tilde{M}_t) = \tilde{w}(M_t)$}
        	\State $\texttt{Out} \gets M_t$
            \State \Return $\texttt{Out}$
        \EndIf
        \State $p_t \gets \arg\max_{a \in (\tilde{M}_t \setminus M_t) \cup (M_t \setminus \tilde{M}_t)}\rad_t(a)$
        \State Pull arm $p_t$
        \State Update empirical means $\bar{\mathbf{u}}_{t+1}$ using the observed reward
        \State $T_{t+1}(p_t) \gets T_t(p_t)+1$
        \State $T_{t+1} \gets T_t(a)\ \forall a \neq p_t$
\EndFor
\end{algorithmic}
\end{algorithm}

\section{Proofs}\label{sec:proofs}

\begin{theorem}[Chen et al. 2014] \label{thm:weak_only} Given any $\delta \in (0,1)$, any decision class $\mathcal{M}\subseteq 2^{[n]}$, and any expected rewards $\mathbf{u} \in \mathbb{R}^n$, assume that the reward distribution $\varphi_a$ for each arm $a\in[n]$ has mean $u(a)$ with an $\sigma$-sub-Gaussian tail. Let $M_*=\arg\max_{M\in\mathcal{M}}w(M)$ denote the optimal set. Set $\rad_t(a)=\sigma\sqrt[]{2\log\left(\frac{4nt^3}{\delta}/T_t(a)\right)}$ for all $t>0$ and $a\in[n]$. Then, with probability at least $1-\delta$, the SWAP algorithm with only weak pulls returns the optimal set $\texttt{Out}=M_*$ and
\begin{align} \label{eq:weak_theorembound}
T\leq O\left(\sigma^2\width(\mathcal{M})^2\mathbf{H}\log(nR^2\mathbf{H}/\delta)\right)
\end{align}
where $T$ denotes the number of samples used by the SWAP algorithm, $\mathbf{H}$ is defined in Eq.\ref{eq:hardness}.
\end{theorem}

In this section, we formally prove the theorems discussed in our paper. Some lemmas we show directly feed from \citet{chen2014combinatorial}'s paper.
\subsection{Strong Arm Pull Problem}
The following maps to Lemma 8 in \citet{chen2014combinatorial}.
\begin{lemma}
\label{l:8}
Suppose that the reward distribution $\varphi_a$ is a $\sigma$-sub-Gaussian distribution for all $a\in[n]$. And if, for all $t > 0$ and all $a\in[n]$, the confidence radius $\rad_t(a)$ is given by
$$
\rad_t(a)=\sigma\sqrt[]{\frac{2\log\left(\frac{4nt^3j^3}{\delta}\right)}{T_t(a)}}
$$
where $T_t(a)$ is the number of samples of arm $a$ up to round $t$. Since $s>1$ the number of samples in a single strong pull will be $s$ each with cost $j$. Then, we have
$$
\Pr \left[\bigcap^{\infty}_{t=1}\xi_t\right] \geq 1-\delta.
$$
\end{lemma}

\begin{proof}
Fix any $t>0$ and $a\in[n]$. Note that $\varphi_a$ is a $\sigma$-sub-Gaussian tail distribution with mean $w(a)$ and $\bar{w}_t(a)$ is the empirical mean of $\varphi_a$ from $T_t(a)$ samples.
\begin{subequations}
\label{eq:optim}
\begin{align}
&\Pr\left[|\bar{w}_t(a)-w_t(a)| \geq \sigma\sqrt[]{\frac{2\log\left(\frac{4nt^3j^3}{\delta}\right)}{T_t(a)}} \right] \nonumber \\
&=\sum_{b=1}^{t-1} \Pr\left[|\bar{w}_t(a)-w_t(a)| \geq \sigma\sqrt[]{\frac{2\log\left(\frac{4nt^3j^3}{\delta}\right)}{bs}}, T_t(a)=bs \right] \label{eq:l8t}\\
& \leq \sum_{b=1}^{t-1} 2\exp\left(\frac{-bs\left(\sigma\sqrt[]{\frac{2\log\left(\frac{4nt^3j^3}{\delta}\right)}{bs}}\right)^2}{2\sigma^2}\right) \label{eq:l8hoeffding} \\
& = \sum_{b=1}^{t-1} \frac{\delta}{2nt^3j^3} \nonumber \\
& \leq \frac{\delta}{2nt^2j^3} \label{eq:l8unionbound}
\end{align}
\end{subequations}
where Eq.\ref{eq:l8t} follows from the fact that $1\leq T_t(a)/s \leq t-1$ and Eq.\ref{eq:l8hoeffding} follows from Hoeffding's inequality.  By a union bound over all $a\in [n]$, we see that $\Pr[\xi_t]\geq 1-\frac{\delta}{2t^2j^3}$. Using a union bound again over all $t>0$, we have
\begin{align*}
\Pr\left[\bigcap_{t=1}^{\infty}\xi_t\right] & \geq 1 - \sum_{t=1}^{\infty}\Pr[\neg\xi_t] \\
& \geq 1 - \sum_{t=1}^{\infty}\frac{\delta}{2t^2j^3} \\
& = 1 - \frac{\pi^2}{12j^3}\delta \\
& \geq 1-\delta
\end{align*}
\end{proof}

The rest of the lemmas in \citet{chen2014combinatorial}'s paper hold. We can now prove Theorem \ref{thm:strong_only_ap}

\begin{theorem} \label{thm:strong_only_ap}
 Given any $\delta \in (0,1)$, any decision class $\mathcal{M}\subseteq 2^{[n]}$, and any expected rewards $\mathbf{w} \in \mathbb{R}^n$, assume that the reward distribution $\varphi_a$ for each arm $a\in[n]$ has mean $w(a)$ with an $\sigma$-sub-Gaussian tail. Let $M_*=\arg\max_{M\in\mathcal{M}}w(M)$ denote the optimal set. Set $\rad_t(a)=\sigma\sqrt[]{2\log\left(\frac{4nt^3j^3}{\delta}/T_t(a)\right)}$ for all $t>0$ and $a\in[n]$. Then, with probability at least $1-\delta$, the CLUCB algorithm with only strong pulls where $j\geq1$ and $s>j$ returns the optimal set $\texttt{Out}=M_*$ and
\begin{align} \label{theorembound}
T\leq O\left(\frac{\sigma^2\width(\mathcal{M})^2\mathbf{H}\log(nj^3R^2\mathbf{H}/\delta)}{s}\right)
\end{align}
where $T$ denotes the number of samples used by the CLUCB algorithm, $\mathbf{H}$ is defined in Eq.\ref{eq:hardness}.
\end{theorem}

\begin{proof}
Lemma \ref{l:8} indicates that the event $\xi \triangleq \bigcap_{t=1}^{\infty}\xi_t$ occurs with probability at least $1-\delta$. In the rest of the proof, we shall assume that this event holds.

By using Lemma 9 from \citet{chen2014combinatorial} and the assumption on $\xi$, we see that $\texttt{Out}=M_*$. Next, we focus on bounding the total number of $T$ samples.

Fix any arm $a\in[n]$. Let $T(a)$ denote the total information gained from pulling arm $a\in[n]$. Let $t_a$ be the last round which arm $a$ is pulled, which means that $p_{t_a}=e$. It is easy to see that $T_{t_a}(a)=T(a)-s$. By Lemma 10 from chen et. al., we see that $\rad_{t_a}\geq\frac{\Delta_a}{3\width(\mathcal{M})}$. Using the definition of $\rad_{t_a}$, we have
\begin{equation}
\label{eq:thm_strong_1}
\frac{\Delta_a}{3\width(\mathcal{M})}
\leq \sigma \sqrt[]{\frac{2\log(4nt_{a}^{3}j^3/\delta)}{T(a)-s}}
\leq \sigma \sqrt[]{\frac{2\log(4nT^{3}j^3/\delta)}{T(a)-s}}.
\end{equation}
By solving Eq.\ref{eq:thm_strong_1} for $T(a)$, we obtain
\begin{equation} \label{eq:thm_strong_arm_bound}
T(a)\leq \frac{18\width(\mathcal{M})^2\sigma^2}{\Delta_{a}^{2}} \log(4nT^3j^3/\delta)+s
\end{equation}
Define $\Htilde=\max\{\width(\mathcal{M})^2\sigma^2\mathbf{H},1\}$. Using similar logic to \citet{chen2014combinatorial} and the fact that the information gained per pull is $s$, we show that
\begin{equation} \label{eq:thm_strong_T_bound}
T\leq \frac{499\Htilde\log(4nj^3\Htilde/\delta)}{s} +2n
\end{equation}
Theorem \ref{thm:strong_only} follows immediately from Eq. \ref{eq:thm_strong_T_bound}.

If $n\geq \frac{1}{2}T$, then $T \leq 2n$ and Eq. \ref{eq:thm_strong_T_bound} holds. For the second case we assume $n < \frac{1}{2} T$.  Since $T>n$, we write
\begin{equation} \label{eq:thm_strong_CT}
T = \frac{C \Htilde\log\left(4nj^3\Htilde/\delta\right)}{s}+n,\text{ for some }C>0.
\end{equation}
If $C<499$, then Eq. \ref{eq:thm_strong_T_bound} holds. Suppose, on the contrary, that $C > 499$. We know that $T=\frac{1}{s}\sum_{a\in [n]}T(a)$. Using this fact and summing Eq. \ref{eq:thm_strong_arm_bound} for all $a\in [n]$, we have
\begin{eqnarray}
T & \leq & \frac{1}{s} \left(ns + \sum_{a\in[n]} \frac{18\width(\mathcal{M})^2\sigma^2}{\Delta_a^2} \log(4nj^3T^3/\delta) \right) \nonumber \\
& \leq & n + \frac{18\Htilde\log(4nj^3T^3/\delta)}{s} \nonumber \\
& = & n + \frac{18\Htilde\log(4nj^3/\delta)}{s} + \frac{54\Htilde\log(T)}{s} \nonumber \\
& \leq & n + \frac{18\Htilde\log(4nj^3/\delta)}{s}  \nonumber \\
& & + \frac{54\Htilde\log(2C\Htilde\log(4nj^3\Htilde/\delta))}{s} \label{eq:thm_strong_reason1} \\
& = & n + \frac{18\Htilde\log(4nj^3/\delta)}{s} \nonumber \\
& & + \frac{54\Htilde\log(2C)}{s} + \frac{54\Htilde\log(\Htilde)}{s} \nonumber \\
& & +  \frac{54\Htilde\log\log(4nj^3\Htilde/\delta)}{s} \nonumber \\
& \leq & n + \frac{18\Htilde\log(4nj^3\Htilde/\delta)}{s} \nonumber \\ 
& & + \frac{54\Htilde\log(2C)\log(4nj^3\Htilde/\delta)}{s} \nonumber \\
& & + \frac{54\Htilde\log(4nj^3\Htilde/\delta)}{s} + \frac{54\Htilde\log(4nj^3\Htilde/\delta)}{s} \label{eq:thm_strong_reason2} \\
& = & (126 + 54\log(2C))\frac{\Htilde\log(4nj^3\Htilde/\delta)}{s} \nonumber \\
& < & n + \frac{C\Htilde\log(4nj^3\Htilde/\delta)}{s} \label{eq:thm_strong_reason3} \\
& = & T , \label{eq:thm_strong_contra}
\end{eqnarray}
where Eq. \ref{eq:thm_strong_reason1} follows from Eq. \ref{eq:thm_strong_CT} and the assumption that $n < \frac{1}{2}T$; Eq. \ref{eq:thm_strong_reason2} follows from $\Htilde \geq 1$, $j \geq 1$, and $\delta < 1$; Eq. \ref{eq:thm_strong_reason3} follows since $126+54\log(2C)<C$ for all $C > 499$; and Eq. \ref{eq:thm_strong_contra} is due to Eq. \ref{eq:thm_strong_CT}. So Eq. \ref{eq:thm_strong_contra} is a contradiction. Therefore $C \leq 499$ and we have proved Eq. \ref{eq:thm_strong_T_bound}.
\end{proof}

\begin{corollary}\label{cor:ap_strong_better_than_weak}
SWAP with only strong pulls is equally or more efficient than SWAP with only weak pulls when $s>0$ and $0<j\leq C^{\frac{s}{3}-\frac{1}{3}}$ where $C=4n\tilde{\mathbf{H}}/\delta$.
\end{corollary}
\begin{proof}
\begin{align}
T_{strong} &\leq T_{weak} \nonumber \\
\frac{499\tilde{\mathbf{H}}\log(4nj^3\tilde{\mathbf{H}}/\delta)}{s} + 2n
&\leq 499\tilde{\mathbf{H}}\log(4nj^3\tilde{\mathbf{H}}/\delta) + 2n \nonumber \\
\frac{\log(Cj^3)}{s} &\leq \log(C) \label{eq:strong_weak}
\end{align}
Solving for Eq.\ref{eq:strong_weak} we get $s>0$ and $0<j\leq C^{\frac{s}{3}-\frac{1}{3}}$.
\end{proof}

\subsection{Strong Weak Arm Pull (SWAP)}
The following corresponds to Lemma 8 in work by the \citet{chen2014combinatorial}.
\begin{lemma} \label{lemma:swap_8}
Suppose that the reward distribution $\varphi_a$ is a $\sigma_1$-sub-Gaussian distribution for all $a\in[n]$. For all $t > 0$ and all $a\in[n]$, the confidence radius $\rad_t(a)$ is given by
$$
\rad_t(a)=\sigma_1\sqrt[]{\frac{2\log\left(\frac{4nCost_t^3}{\delta}\right)}{T_t(a)}}
$$
where $T_t(a)$ is the number of samples of arm $a$ up to round $t$. Since $s>1$, the number of samples in a single strong pull are $s$ each with cost $j$. Then, we have
$$
\Pr \left[\bigcap^{\infty}_{t=1}\xi_t\right] \geq 1-\delta.
$$
\end{lemma}

\begin{proof}
Fix any $t>0$ and $a\in [n]$. Note that $\varphi_a$ is $\sigma_1$-sub-Gaussian tail distribution with mean $w(a)$ and $\bar{w}(a)$ is the empirical mean of $\varphi_a$ from $T_t(a)$ samples. Then we have
\begin{align}
& \Pr\left[|\bar{w}_t(a)-w_t(a)| \geq \sigma_1\sqrt[]{\frac{2\log\left(\frac{4nCost_t^3}{\delta}\right)}{T_t(a)}} \right] \\
&=\sum_{b=1}^{t-1} \Pr\left[|\bar{w}_t(a)-w_t(a)| \geq \sigma_1\sqrt[]{\frac{2\log\left(\frac{4nCost_t^3}{\delta}\right)}{Gain_b}}\right] \label{eq:lem_swap_t}\\
& \leq \sum_{b=1}^{t-1} 2\exp\left(\frac{-Gain_b\left(\sigma_1\sqrt[]{\frac{2\log\left(\frac{4nCost_t^3}{\delta}\right)}{Gain_b}}\right)^2}{2R^2}\right) \label{eq:lem_swap_hoeffding} \\
& = \sum_{b=1}^{t-1} \frac{\delta}{2n AvCost^3t^3} \nonumber \\
& \leq \frac{\delta}{2nt^2 AvCost^3} \label{eq:lem_swapunionbound}
\end{align}
where $AvCost$ equal to the average cost until time $t$. Eq.\ref{eq:lem_swap_t} follows from $1\leq T_t(a)/Gain_t \leq t-1$ and Eq.\ref{eq:lem_swap_hoeffding} follows from Hoeffding's inequality.  By a union bound over all $a\in [n]$, we see that $\Pr[\xi_t]\geq 1-\frac{\delta}{2t^2AvCost_t^3}$. Using a union bound again over all $t>0$, we have
\begin{align*}
\Pr\left[\bigcap_{t=1}^{\infty}\xi_t\right] & \geq 1 - \sum_{t=1}^{\infty}\Pr[\neg\xi_t] \\
& \geq 1 - \sum_{t=1}^{\infty}\frac{\delta}{2t^2AvCost^3} \\
& = 1 - \frac{\pi^2}{12AvCost^3}\delta \\
& \geq 1-\delta
\end{align*}
\end{proof}

Given that the rest of the lemmas in the \citet{chen2014combinatorial} paper hold, we now prove the main theorem of our paper.

\begin{theorem}
Given any $\delta_1,\delta_2,\delta_3 \in (0,1)$, any decision class $\mathcal{M}\subseteq 2^{[n]}$ and any expected rewards $\mathbf{w} \in \mathbb{R}^n$, assume that the reward distribution $\varphi_a$ for each arm $a\in [n]$ has mean $w(a)$ with an $\sigma_1$-sub-Gaussian tail. Let $M_*=\arg\max_{M\in\mathcal{M}}w(M)$ denote the optimal set. Set $\rad_t(a)=\sigma_1\sqrt[]{2\log\left(\frac{4nCost_t^3}{\delta}/T_t(a)\right)}$ for all $t>0$ and $a\in[n]$, set $\epsilon_1=\sigma_2\sqrt[]{2\log\left(\frac{1}{2}\delta_2/T\right)}$, and set $\epsilon_2=\sigma_3\sqrt[]{2\log\left(\frac{1}{2}\delta_3/n\right)}$. Then, with probability at least $(1-\delta_1)(1-\delta_2)(1-\delta_3)$, the SWAP algorithm (Algorithm \ref{alg:swap}) returns the optimal set $\texttt{Out}=M_*$ and
\begin{equation}
\label{eq:swap_thm}
T\leq O\left(\frac{R^2\width(\mathcal{M})^2\mathbf{H}\log\left(nR^2\left(\bar{X}_{Cost}-\epsilon_1\right)^3\mathbf{H}/\delta\right)}{\bar{X}_{Gain}-\epsilon_2}\right),
\end{equation}
where $T$ denotes the number of samples used by Algorithm \ref{alg:swap}, $\mathbf{H}$ is defined in Eq. \ref{eq:hardness} and $\width(\mathcal{M})$ is defined by \citet{chen2014combinatorial}.
\end{theorem}

\begin{proof}
Lemma \ref{lemma:swap_8} indicates that the event $\xi \triangleq \bigcap_{t=1}^{\infty}\xi_t$ occurs with probability at least $1-\delta$. In the rest of the proof, we assume that this event holds.

Using Lemma 9 from \citet{chen2014combinatorial} and the assumption on $\xi$, we see that $\texttt{Out}=M_*$. Next, we bound the total number of $T$ samples.

Fix any arm $a\in[n]$. Let $T(a)$ denote the total information gained from pulling arm $a\in[n]$. Let $t_a$ be the last round which arm $a$ is pulled, which means that $p_{t_a}=a$. Trivially, $T_{t_a}(a)=T(a)-s$. By Lemma 10 from \citet{chen2014combinatorial}, we see that $\rad_{t_a}\geq\frac{\Delta_a}{3\width(\mathcal{M})}$. Using the definition of $\rad_{t_a}$, we have
\begin{align}
\label{eq:thm_swap_1}
\frac{\Delta_a}{3\width(\mathcal{M})} & \leq R\sqrt[]{\frac{2\log(4nCost_{t_a}^3/\delta)}{T(e)-Gain_{t_a}}} \nonumber \\
& \leq R\sqrt[]{\frac{2\log(4nCost_{T}^3/\delta)}{T(a)-Gain_{t_a}}}.
\end{align}
Solving for $T(a)$ in Eq. \ref{eq:thm_swap_1} we get
\begin{equation}
\label{eq:thm_swap_te_bound1}
T(a) \leq \frac{18\width(\mathcal{M})^2R^2}{\Delta_e^2}\log(4nCost_T^3/\delta) + Gain_{t_a}
\end{equation}

Define $\bar{X}_{Cost}=\mathbb{E}[Cost]$ as the expected cost of pulling an arm. Since we strong pull an arm with probability $\alpha=\frac{s-j}{s-1}$, we know
\begin{equation}
\bar{X}_{Cost}=\mathbb{E}[Cost_T]=\alpha j +(1-\alpha).
\end{equation}
Define $X_{Cost_t}$ as the cost of pulling an arm at time $t$. Assuming that each random variable $X_{Cost_t}$ is $R_1$-sub-Gaussian we can write the following using the Hoeffding inequality,
\begin{equation}
\Pr\left(\left|\frac{1}{T}\sum_{t=1}^{T}C_{Cost_t}-\bar{X}_{Cost}\right| \geq \epsilon_1\right) \leq 2\exp\left(-\frac{T\epsilon_1^2}{2R_1}\right)
\end{equation}
If we set $\epsilon_1=R_1\sqrt[]{2log(\frac{1}{2}\delta_2)/T}$ then with probability $(1-\delta_2)$
\begin{equation}
\label{eq:thm_swap_cost_bound}
\frac{Cost_T}{T} \in \left(\bar{X}_{Cost}-\epsilon_1,\bar{X}_{Cost}+\epsilon\right).
\end{equation}
Combining Eq. \ref{eq:thm_swap_te_bound1} and Eq. \ref{eq:thm_swap_cost_bound} we get
\begin{equation}
\label{eq:thm_swap_te_bound2}
T(e) \leq \frac{18\width(\mathcal{M})^2R^2}{\Delta_e^2}\log(4n(\bar{X}_{Cost}-\epsilon_1)^3T^3/\delta) + Gain_{t_e}
\end{equation}

Define $\bar{X}_{Gain}=E[Gain]$ as the expected information gain from pulling an arm. Since we pull an arm with probability $\alpha$, we know that
\begin{equation}
\bar{X}_{Gain}=E[Gain]=\alpha s + (1-\alpha)
\end{equation}
Define $X_{Gain_t}$ as the information gain of pulling an arm at time $t$. Assuming that each random variable $X_{Gain_t}$ is $R_2$-sub-Gaussian we can write the following using the Hoeffding inequality.
\begin{equation}
\Pr\left(\left|\frac{1}{n}\sum_{e\in[n]}Gain_{t_e}-\bar{X}_{Gain}\right| \geq \epsilon_2\right) \leq 2\exp\left(\frac{-n\epsilon_2^2}{2R_2^2}\right)
\end{equation}
If we set $\epsilon_2=R_2\sqrt[]{2log(\frac{1}{2}\delta_3)/n}$ then with probability $(1-\delta_2)$
\begin{equation}
\label{eq:thm_swap_gain_bound}
\frac{\sum_{e\in[n]}Gain_{t_e}}{n} \in \left(\bar{X}_{Gain}-\epsilon_2,\bar{X}_{Gain}+\epsilon_2\right).
\end{equation}
Similarly to the proof for Theorem \ref{thm:strong_only}, define $\Htilde=\max\{\width(\mathcal{M})^2R^2\mathbf{H},1\}$. In the rest of the proof we will show that
\begin{equation}
\label{eq:thm_swap_T_bound}
T \leq \frac{499\Htilde\log\left(4n\left(\bar{X}_{Cost}+\epsilon_1\right)^3\Htilde/\delta\right)}{\bar{X}_Gain-\epsilon_2} + 2n
\end{equation}
Notice that theorem follows immediately from Eq. \ref{eq:thm_swap_T_bound}.

If $n\geq \frac{1}{2}T$, then Eq. \ref{eq:thm_swap_T_bound} holds. Let's then assume that $n<\frac{1}{2}T$. Since $T>n$, we can write
\begin{equation}
\label{eq:thm_swap_T_eq}
T=\frac{C\Htilde\log(4n(X_{Cost}+\epsilon_1)^3\Htilde/\delta}{\bar{X}_{Gain}-\epsilon_2} + n
\end{equation}
If $C \leq 499$ then Eq. \ref{eq:thm_swap_T_bound} holds. Suppose then that $C > 499$. Notice that $T=\sum_{a\in[n]}T(a)/Gain_{t_a}$. By summing up Eq. \ref{eq:thm_swap_te_bound2} for all $a\in[n]$ we have
\begin{eqnarray}
T & \leq & n + \sum_{a\in[n]}\frac{18\width(\mathcal{M})^2R^2\log(4n(\bar{X}_{Cost}+\epsilon_1)T^3/\delta}{\Delta_a^2Gain_{t_a}} \nonumber \\
& \leq & n + \frac{18\Htilde\log(4n(\bar{X}_{Cost}+\epsilon_1)^3T^3/\delta)}{\bar{X}_{Gain}-\epsilon_2} \label{eq:thm_swap_reason1}\\
& = & n + \frac{18\Htilde\log(4n(\bar{X}_{Cost}+\epsilon_1)^3/\delta)}{\bar{X}_{Gain}-\epsilon_2} + \frac{54\Htilde\log(T)}{\bar{X}_{Gain}-\epsilon_2} \nonumber \\
& \leq & n + \frac{18\Htilde\log(4n(\bar{X}_{Cost}+\epsilon_1)^3/\delta)}{\bar{X}_{Gain}-\epsilon_2} \nonumber \\
& & + \frac{54\Htilde\log(2c\Htilde\log(4n(\bar{X}_{Cost}-\epsilon_1)^3\Htilde/\delta))}{\bar{X}_{Gain}-\epsilon_2} \label{eq:thm_swap_reason2}\\
& = &  n + \frac{18\Htilde\log(4n(\bar{X}_{Cost}+\epsilon_1)^3/\delta)}{\bar{X}_{Gain}-\epsilon_2} + \frac{54\Htilde\log(2C)}{\bar{X}_{Gain}-\epsilon_2} \nonumber \\
& & + \frac{54\Htilde\log(\Htilde)}{\bar{X}_{Gain}-\epsilon_2} \nonumber \\
& & + \frac{54\Htilde\log\log(4n(\bar{X}_{Cost}+\epsilon_1)^3\Htilde/\delta)}{\bar{X}_{Gain}-\epsilon_2} \nonumber \\
& \leq & n + \frac{18\Htilde\log(4n(\bar{X}_{Cost}+\epsilon_1)^3\Htilde/\delta)}{\bar{X}_{Gain}-\epsilon_2} \nonumber \\
& & + \frac{54\Htilde\log(2C)\log(4n(\bar{X}_{Cost}+\epsilon_1)^3\Htilde/\delta)}{\bar{X}_{Gain}-\epsilon_2} \nonumber \\
& & + \frac{54\Htilde\log(4n(\bar{X}_{Cost}+\epsilon_1)^3\Htilde/\delta)}{\bar{X}_{Gain}-\epsilon_2} \nonumber \\
& & + \frac{54\Htilde\log(4n(\bar{X}_{Cost}+\epsilon_1)^3\Htilde/\delta)}{\bar{X}_{Gain}-\epsilon_2} \label{eq:thm_swap_reason3} \\
& = & n + (126+54\log(2C))\frac{\Htilde\log(4n(\bar{X}_{Cost}+\epsilon_1)^3\Htilde/\delta)}{\bar{X}_{Gain}-\epsilon_2} \nonumber \\
& < & n + \frac{C\Htilde\log(4n(\bar{X}_{Cost}+\epsilon_1)^3\Htilde/\delta)}{\bar{X}_{Gain}-\epsilon_2} \label{eq:thm_swap_reason4}\\
& = & T, \label{eq:thm_swap_contra} 
\end{eqnarray}
where Eq. \ref{eq:thm_swap_reason1} follows from Eq. \ref{eq:thm_swap_gain_bound}; Eq. \ref{eq:thm_swap_reason2} follows from Eq. \ref{eq:thm_swap_T_eq} and the assumption $n < \frac{1}{2}T$; Eq. \ref{eq:thm_swap_reason3} follows from $\Htilde \geq 1$, $\delta < 1$, and $\bar{X}_{Cost}+\epsilon \geq 1$; Eq. \ref{eq:thm_swap_reason4} follows since $126+54\log(2C)<C$ for all $C>499$; and Eq. \ref{eq:thm_swap_contra} is due to Eq. \ref{eq:thm_swap_T_eq}. So Eq. \ref{eq:thm_swap_contra} is a contradiction. Therefore $C \leq 499$ and we have proved Eq. \ref{eq:thm_swap_T_bound}.
\end{proof}

\section{Additional Details about the Admissions Decisions Classifier}\label{app:classifier}
    To effectively model the graduate admissions process, we needed a way to accurately represent whether a particular applicant will be admitted to the program. Using 3 years of previous admissions data, including letters of recommendation, we built a classifier modeling the graduate chair's decision for a particular applicant. The classifier's accuracy can be found in Table \ref{table:predictor}.
    
\begin{table}[h]
\centering
\begin{tabular}{|| l || r | r | r ||} 
 \hline
 Type & \% Correct & Precision & Recall \\ [0.5ex] 
 \hline\hline
 Ph.D. & 77.8\% & 61.1\% & 39.7\% \\ 
 Masters & 89.2\% & 13.1\% & 55.3\% \\
 Total & 85.5\% & 33.5\% & 42.0\% \\[1ex] 
 \hline
\end{tabular}
\caption{Current predictor results on the testing data}
\label{table:predictor}
\end{table}

Some general features from the application are GPA, GRE scores, TOEFL scores, area of interest (Machine Learning, Theory, Vision, and so on), previous degrees, and universities attended. We included country of origin since the nature of applications may vary in different regions due to cultural norms. Another basic feature included was sex. We included this to check if the classifier picked up on any biased decision making (with sex and region).

Other features were generated from automatically processing the recommendation letters. Text from the letters was pulled from pdfs and OCR for scanned letters. We then cleaned the raw text with NLTK, removing stop words and stemming text \cite{Bird06:NLTK}. One feature we chose was the length of recommendation letter, chosen after polling the admissions committee on what they thought would be important. \citet{schmader2007} used Latent Dirichlet Allocation (LDA) to find word groups in recommendation letters for Chemistry and Biochemistry students \cite{BleiLDA}. Their five word groups included standout words (excellen*, superb, outstanding etc.), ability words ( talent*, intell*, smart*, skill*, etc.), grindstone words (hardworking, conscientious, depend*, etc.), teaching words (teach, instruct, educat*, etc.), and research words (research*, data, study, etc.). We found that these word groups translated well to Computer Science students. Important words for acceptance were research words, standout words, and ability words. Letters that only included words from the teaching word group indicated a less useful recommendation letter. We used counts of the various word groups as a feature in the classifier.

\section{Additional Experimental Results}\label{app:experiments}

\subsection{Gaussian Experiments}

While running SWAP, we first compare where the general, varied-cost version of SWAP is better than SWAP with strong pulls only (Figure \ref{figure:better_than_strong}) and where it is better than SWAP with only weak pulls (Figure \ref{figure:better_than_weak}). We then noticed that there should be an optimal zone where the general version of SWAP would perform better than both of the trivial cases.

Both graphs examine the symmetric difference between the average cost values of SWAP and either Strong or Weak Pull only with different parameter values of $s$ and $j$. 
\begin{figure} 
\includegraphics[width=\columnwidth]{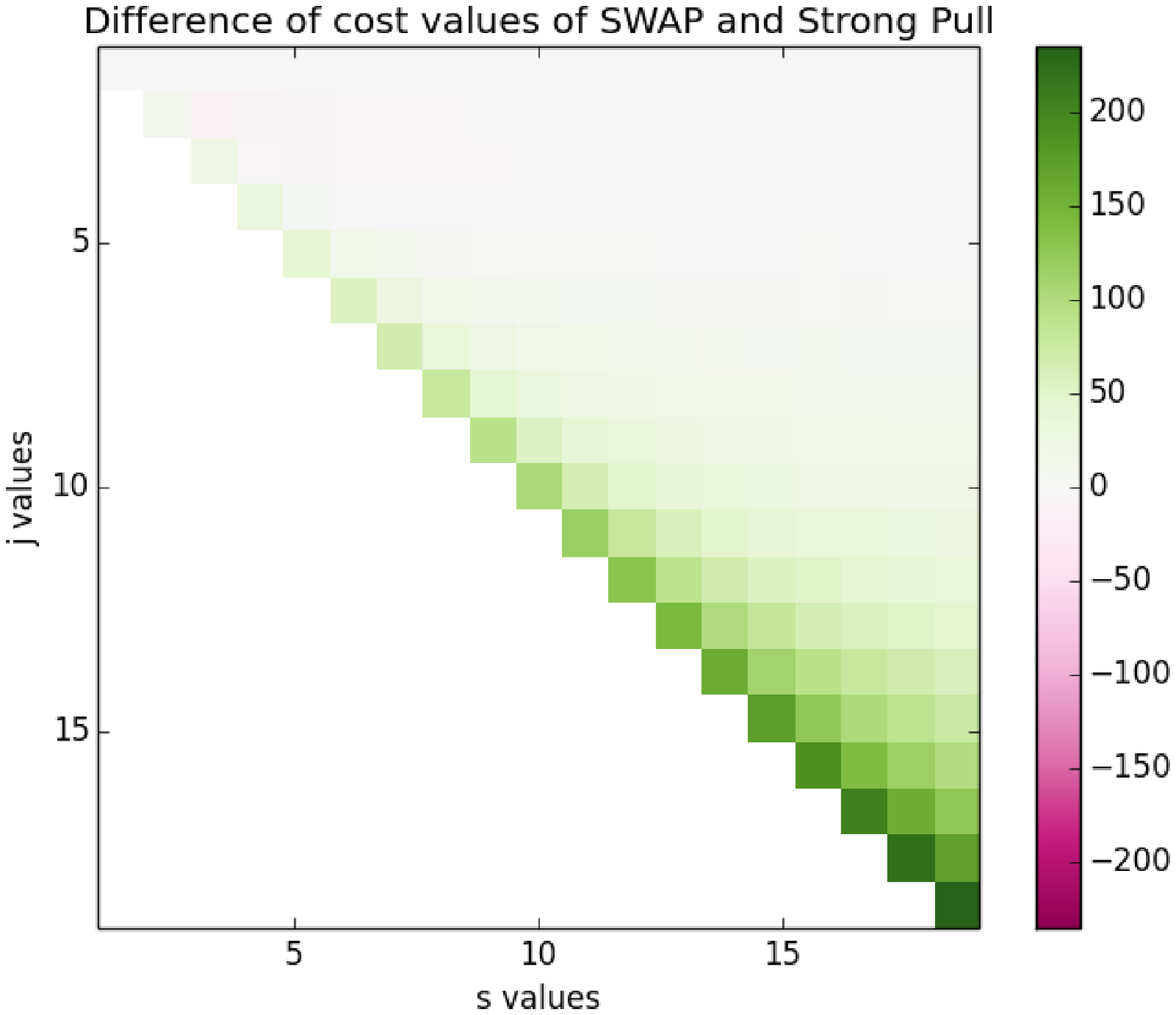}
\caption{Heat map showing where SWAP is better than Strong Pull Only.}
\label{figure:better_than_strong}
\includegraphics[width=\columnwidth]{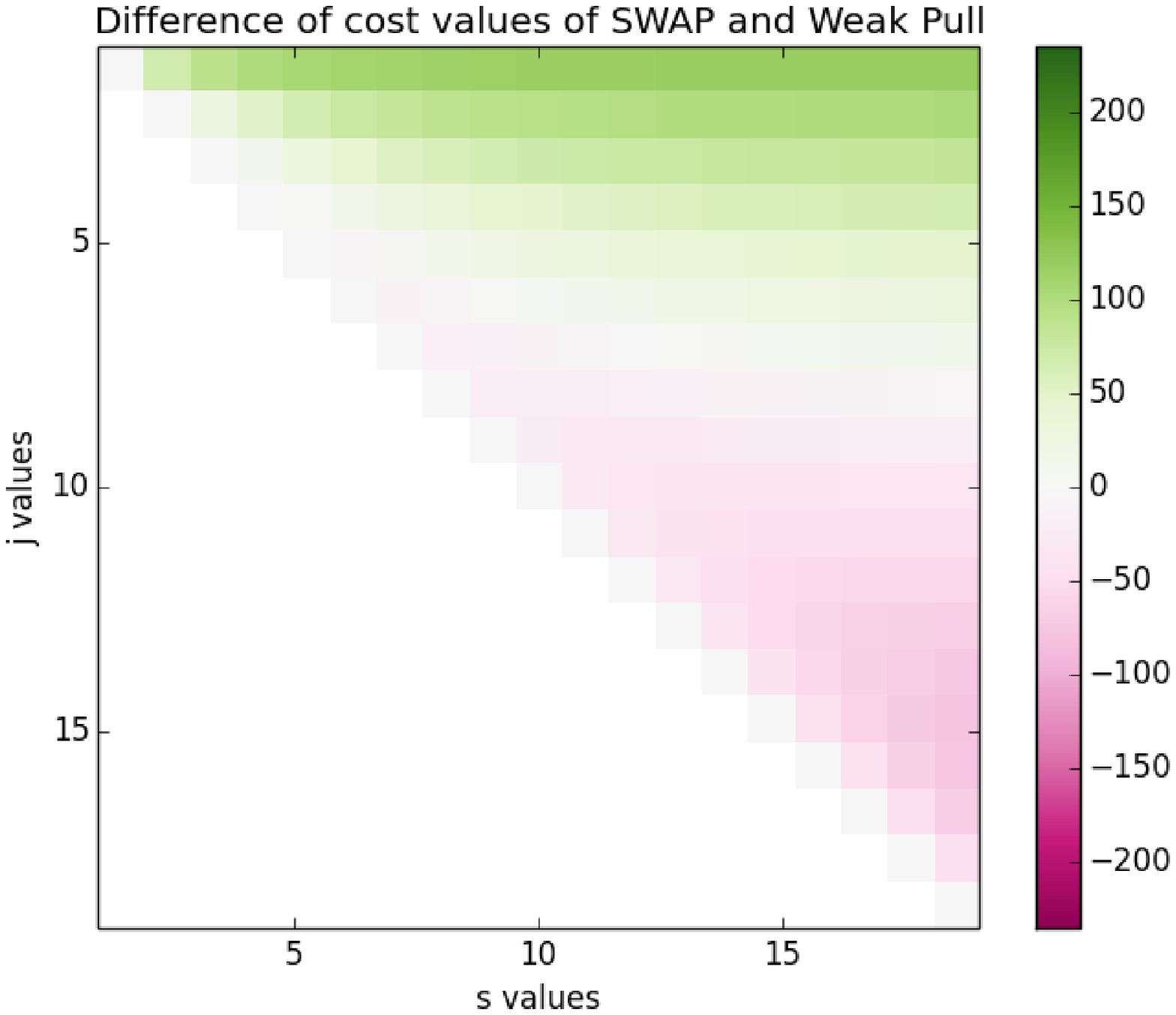}
\caption{Heat map showing where SWAP is better than Weak Pull Only.}
\label{figure:better_than_weak}
\end{figure}

\subsection{Graduate Admissions Experiment}

We ran SWAP over both Masters and Ph.D. students over various values of $s$ (Figure \ref{figure:cost_over_s}). The total cost of running these experiments aligns with the resources spent during the actual admissions decision process. 

\begin{figure}
\includegraphics[width=\columnwidth]{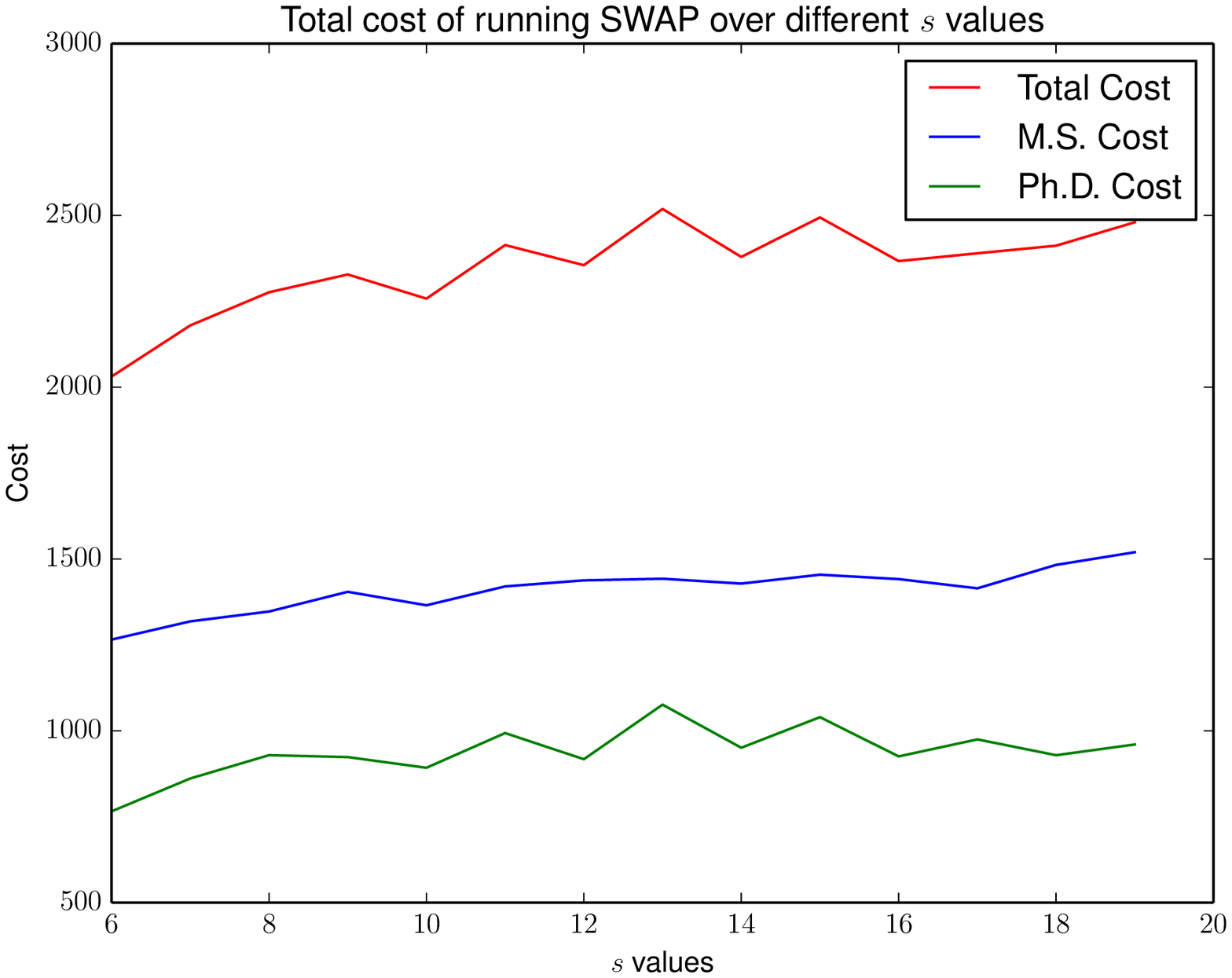}
\caption{Total cost of running SWAP over different $s$ values}
\centering
\label{figure:cost_over_s}
\end{figure}

When running SWAP experiments to formally promote diversity, one experiment not listed in the main paper was testing our diverse SWAP algorithm over an applicant's main choice of research area (Table \ref{table:area}). In practice, the applicants accepted already had a high diversity utility in regards to research area. SWAP slightly increased this diversity utility.

\begin{table}[h]
\centering
\begin{tabular}{ l | r r  r } 
  & General& Diversity \\ [0.7ex] 
 \hline\\
 SWAP & 8.3 (0.03)& 32.5 (0.03) \\ 
 Actual & 8.6 & 27.4 \\[1ex] 
\end{tabular}
\caption{SWAP's average gain in reported area of study diversity over our actual acceptances. The first column shows general fit utility and the second diversity utility The standard deviation over the experiments of SWAP can be found in parentheses.}
\label{table:area}
\end{table}